\documentclass[pmlr]{jmlr}%

\usepackage{longtable}%

 \usepackage{booktabs}
 
\usepackage[load-configurations=version-1]{siunitx} %

\makeatletter
\def\set@curr@file#1{\def\@curr@file{#1}} %
\makeatother

\theorembodyfont{\upshape}
\theoremheaderfont{\scshape} 
\theorempostheader{:}
\theoremsep{\newline}

\graphicspath{{assets/}}

\usepackage{eucal}

\usepackage{bbm}

\newcommand\pa[1]{\ensuremath{\left( #1 \right)}} %
\newcommand\pb[1]{\ensuremath{\left[ #1 \right]}} %
\newcommand\pc[1]{\ensuremath{\left\{ #1 \right\}}} %
\newcommand\R{\ensuremath{\mathbb{R}}} %
\newcommand\E{\ensuremath{\mathbb{E}}} %
\newcommand\Prob{\ensuremath{\mathbb{P}}} %

\newcommand\sC{\ensuremath{\mathcal{C}}}
\newcommand\sD{\ensuremath{\mathcal{D}}}
\newcommand\sE{\ensuremath{\mathcal{E}}}
\newcommand\sL{\ensuremath{\mathcal{L}}}

\newcommand\sX{\ensuremath{\mathcal{X}}}

\newcommand\sS{\ensuremath{\mathcal{S}}}
\newcommand\sR{\ensuremath{\mathcal{R}}}
\newcommand\sT{\ensuremath{\mathcal{T}}} 
\newcommand\sF{\ensuremath{\mathcal{F}}} 
\newcommand\sO{\ensuremath{\mathcal{O}}}
\newcommand{\norm}[1]{\left\lVert#1\right\rVert}

\DeclareMathOperator*{\argmin}{arg\,min}

\newif\ifhighresfig
\highresfigfalse    %

\usepackage{array}
\newcolumntype{C}[1]{>{\arraybackslash\centering}p{#1}}
\newcolumntype{R}[1]{>{\arraybackslash\raggedleft}p{#1}}
\newcolumntype{L}[1]{>{\arraybackslash\raggedright}p{#1}}
\newsavebox\CBox 
\def\textBF#1{\sbox\CBox{#1}\resizebox{\wd\CBox}{\ht\CBox}{\textbf{#1}}}

\newcommand{\wpq}[1]{}
\newcommand{\nms}[1]{}

\title[Sample-Specific Debiasing for Better Image-Text Models]{Sample-Specific Debiasing for Better Image-Text Models}

\newcommand{\addrmit}{Massachusetts Institute of Technology, Cambridge, MA, USA}
\newcommand{\addrbidmc}{Beth Israel Deaconess Medical Center, Harvard Medical School, Boston, MA, USA}

\author{\Name{Peiqi Wang}$^1$ \Email{wpq@mit.edu}
        \AND
        \Name{Yingcheng Liu}$^1$ \Email{liuyc@mit.edu}
        \AND
        \Name{Ching-Yun Ko}$^1$ \Email{cyko@mit.edu}
        \AND
        \Name{William M. Wells}$^{1}$ \Email{sw@bwh.harvard.edu}
        \AND
        \Name{Seth Berkowitz}$^2$ \Email{sberkowi@bidmc.harvard.edu}
        \AND
        \Name{Steven Horng}$^2$ \Email{shorng@bidmc.harvard.edu}
        \AND
        \Name{Polina Golland}$^1$ \Email{polina@csail.mit.edu}
        \AND
        $^1$ \addr \addrmit \\ 
        $^2$ \addr \addrbidmc
}

\begin{document}

\maketitle

\begin{abstract}
Self-supervised representation learning on image-text data facilitates crucial medical applications, such as image classification, visual grounding, and cross-modal retrieval. One common approach involves contrasting semantically similar (positive) and dissimilar (negative) pairs of data points. Drawing negative samples uniformly from the training data set introduces false negatives, i.e., samples that are treated as dissimilar but belong to the same class. In healthcare data, the underlying class distribution is nonuniform, implying that false negatives occur at a highly variable rate. To improve the quality of learned representations, we develop a novel approach that corrects for false negatives. Our method can be viewed as a variant of debiased contrastive learning that uses estimated sample-specific class probabilities. We provide theoretical analysis of the objective function and demonstrate the proposed approach on both image and paired image-text data sets. Our experiments illustrate empirical advantages of sample-specific debiasing.
\end{abstract}

\section{Introduction}

\nms{High-level feedback: Right now, when I read the intro, I lose the thread of the argument for why you are doing what you are doing, because each paragraph gets lost in the details of the related work. I would recommend either: (1) expanding the overview paragraph at the beginning to cover every step in the logic chain or (2) splitting the related work into its own mini subsection or even separate section. For #1, the "logic chain" seems to be: We want to do self-supervised representation learning on image-text data. $\rightarrow$ We need contrastive learning  $\rightarrow$ Contrastive learning suffers from a false negative problem  $\rightarrow$ Current approaches for dealing with false negatives assume a uniform class distribution, which is not accurate  $\rightarrow$ We propose to do xyz to compensate.  }

In this paper, we propose and demonstrate a novel approach for contrastive learning of image-text representations. Specifically, we propose to estimate sample-specific class probabilities, i.e., the latent class probability of each data point, to appropriately compensate for the effect of false negative samples in the learning procedure. Our method achieves state-of-the-art performance on a wide range of downstream tasks.

\nms{Is ``sample-specific class frequencies" standard terminology in the literature? To someone who doesn't know your method already, it sounds like you're assigning class frequencies for the individual sample itself (as opposed to its potential contrastive partners), which doesn't make sense. What about something like: ``sample-specific contrastive class frequencies"? Only had time to read the intro, so maybe this isn't the best term, but the current wording is confusing.}

\nms{IMO this paragraph is too long and/or the detail is too early in the intro, and pushes the description of your contribution to be too late. Would you consider starting off by just stating that you're working in the setting of image-text, and then at the end of the Intro, after you've made your contributions clear, explain that image-text learning is a very general setup and so your method could be helpful for a variety of tasks including the ones you described here?}

Self-supervised representation learning on paired image-text data uses text as ``labels'', requiring no further annotations beyond what has been routinely documented \citep{radfordLearningTransferableVisual2021}. By leveraging natural language to reference visual concepts and vice versa, the resulting image-text models trained using self-supervised objectives can perform a diverse set of vision-language tasks \citep{radfordLearningTransferableVisual2021, zhangContrastiveLearningMedical2022}.

When applied to the medical domain, image-text models can (i) retroactively label images to select relevant patients for a clinical trial, (ii) help physicians verify the accuracy of a report by noting whether the referred location (i.e., visual grounding of the text) is consistent with their impression of the image, and (iii) enable informed interpretation of medical images by retrieving similar patients from a database. Moreover, the abundance of paired image-text data (e.g., radiographs and radiology reports, histology slides and pathology reports) suggests the broad applicability of self-supervision using image-text data to improve healthcare.

In this paper, we focus on contrastive learning, a self-supervised approach that encourages the representations of semantically similar or \textit{positive} pairs of data to be close, and those of dissimilar or \textit{negative} pairs to be distant. Contrastive learning has been applied in the medical domain, demonstrating impressive transfer capabilities on a diverse set of downstream tasks \citep{chauhanJointModelingChest2020,huangGLoRIAMultimodalGlobalLocal2021,liaoMultimodalRepresentationLearning2021,mullerJointLearningLocalized2022,zhangContrastiveLearningMedical2022,boeckingMakingMostText2022,wangUsingMultipleInstance2023,bannurLearningExploitTemporal2023}. The biggest improvements come from addressing challenges unique to this domain. Examples include using cross-attention to localize areas of interest to handle the lack of effective pathology detectors~\citep{huangGLoRIAMultimodalGlobalLocal2021}, fine-tuning language models on medical corpora to address linguistic challenges in clinical notes~\citep{boeckingMakingMostText2022}. In a similar spirit, our work aims to address the nonuniform class distribution typical of healthcare data. 

\nms{I think the argument that "the biggest improvements come from addressing challenges unique to this domain" and the other examples aren't clarifying anything and are cluttering the intro; would remove.}
\wpq{response: I made the argument shorter.}

\begin{figure}[t] 
    \centering  
    \includegraphics[width=\textwidth]{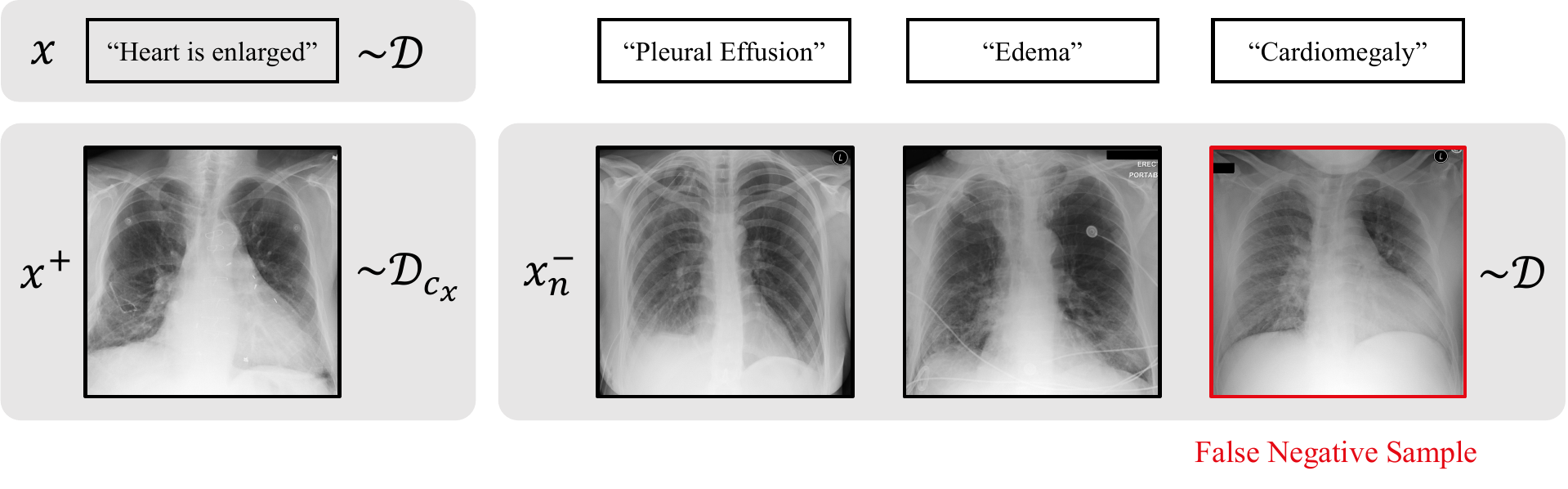} 
    \caption{False negatives in paired image-text data. Drawing negative image samples $x_n^-$ from the data distribution $\sD$ may result in samples that are semantically similar to the text $x$ (``Heart is enlarged" and ``Cardiomegaly" imply the same pathology). False negative samples occur at an uneven rate (e.g., depending on the pathology type) and degrade the performance of image-text models on downstream tasks.}
    \label{fig:teaser_figure}  
\end{figure}

Using text as ``labels'' induces a nonuniform class distribution with large support. Specifically, natural language descriptions of medical images can represent a vast number of possible classes. Most descriptions belong to a few common classes (e.g., cardiomegaly) while the remaining descriptions are spread across many rare classes (e.g., left apical pneumothorax). In our problem of chest X-ray representation learning, each CheXpert label \citep{irvinCheXpertLargeChest2019} representing common classes in chest radiographs is mentioned in $10\text{-}25\%$ of the radiology reports in MIMIC-CXR, a large collection of chest X-ray images and radiology reports \citep{johnsonMIMICCXRDeidentifiedPublicly2019a}. At the same time, the rare classes may be more descriptive and associated with high risk, which makes their accurate identification important in clinical applications. Our goal is to handle the highly nonuniform class distribution that presents a challenge for many existing contrastive learning approaches.

\nms{I don't understand the connection with the previous sentence -- does this mean 75-90\% of reports correspond to rare classes? If so, it's not true that "most descriptions belong to only a small number of common classes"?}
\wpq{response: I will make it clear that 10-25 percent refers to each label.}

When training an image-text model using contrastive learning, each text is positively paired with the associated image from the same imaging event and negatively paired with a batch of images uniformly drawn from the training data set~\citep{liVisualBERTSimplePerformant2019,luViLBERTPretrainingTaskAgnostic2019,chenUNITERUNiversalImageTExt2020a,radfordLearningTransferableVisual2021,liaoMultimodalRepresentationLearning2021,zhangContrastiveLearningMedical2022,wangUsingMultipleInstance2023}. If a negatively paired image is semantically similar to the text, it is considered as a \textit{false negative}~\citep{saunshiTheoreticalAnalysisContrastive2019}, as illustrated in Figure~\ref{fig:teaser_figure}. It has been shown previously that false negatives cause a substantial decline in downstream classification performance when using image representations trained with contrastive learning \citep{chuangDebiasedContrastiveLearning2020}.

One approach to alleviating the problem of false negatives is to
identify false negative pairs explicitly and reduce their effect. In some applications, ground truth class labels are available and can be used to ensure that no false negative pairs are generated~\citep{khoslaSupervisedContrastiveLearning2020,dwibediLittleHelpMy2021}. Unfortunately, deriving categorical labels from text is challenging because it can be difficult to determine the appropriate level of class granularity and ensure sufficient class coverage when the class distribution is nonuniform with large support. Label imputation with nearest neighbor methods~\citep{zhengWeaklySupervisedContrastive2021} or clustering~\citep{chenIncrementalFalseNegative2022} requires non-trivial implementation and additional computational cost. 
In the absence of class labels, a common approach is to treat negative samples whose embeddings are close to positive samples as false negatives and to eliminate them~\citep{huynhBoostingContrastiveSelfSupervised2022, zhangMedicalSymptomDetection2022,zhouDebiasedContrastiveLearning2022,wangImprovingMolecularContrastive2022}. This approach is easy to implement but it overlooks valuable information captured by text if only image embedding is used. Moreover, rejecting negative samples whose embeddings are close to positive samples runs risk of removing valuable ``hard negatives'', i.e., visually similar true negative samples.

Alternatively, debiased contrastive learning uses (possibly) extra positive samples to offset the influence of false negatives without explicitly identifying them~\citep{chuangDebiasedContrastiveLearning2020}. It assumes a uniform class distribution and applies a constant correction to each sample. This strategy can be suboptimal for healthcare data where the class distribution is nonuniform. We observe that naively applying debiased contrastive learning introduces a performance trade-off between coarse-grained tasks (e.g., classifying pneumonia) and fine-grained tasks (e.g., cross-modal retrieval of rare classes), as seen in Table~\ref{fig:mlhc_dcl_fixtaup_tradeoff}.

\wpq{I feel in this paragraph, the goal is to say we solved the FN problem when underlying class is nonuniform and we want to advocate to use the correct sample-specific class probability then show how we can estimate the correction strength in image-text case. The contribution sounds better/more general. Currently these points are not conveyed. according to polina, this is because using the correct probability is obvious and doesnt constitute the major contribution of the paper, which makes sense I guess.}

Our method aims to reduce the effect of false negatives in contrastive representation learning while making no assumption on the underlying class distribution. Specifically, we estimate the class probability (or level of correction) for each data point based on the likelihood of the text provided by a language model.
Our method (i) does not require or attempt to infer class labels, (ii) requires a few extra lines of code to implement when compared with contrastive objectives typically used in practice, (iii) incurs minimal computational overhead,
and (iv) takes advantage of the class information implicitly represented by the text ``labels'' to adaptively mitigate the problem of false negatives in image-text contrastive learning.
We study the advantages of using sample-specific class probability estimates on a small-scale image data set constructed with a nonuniform class distribution. We evaluate our approach on a large set of chest X-ray images and associated radiology reports~\citep{johnsonMIMICCXRDeidentifiedPublicly2019a}, demonstrating superior performance on image classification, visual grounding, and cross-modal retrieval tasks.

\subsection*{Generalizable Insights about Machine Learning in the Context of Healthcare}

Our work (i) provides direct value to those interested in using self-supervised representation learning for healthcare applications, (ii) highlights the importance of considering the distinct characteristics of healthcare data that can pose challenges for methods designed for simpler scenarios (in our case, false negatives degrade the performance of image-text models trained on clinical data), and (iii) shows that a language model provides a useful prior that can be employed in other image-text modeling problems (e.g., brain tumor CT scans with imaging reports) or inspire future research to solve related problems in other data modalities.

\section{Methods}

In this section, we introduce notation and provide a brief overview of debiased contrastive learning \citep{chuangDebiasedContrastiveLearning2020}, followed by the derivation of our method to compensate for potential false negative samples and analysis of the relationship between our approach and the original formulation.

\subsection{Notation and Problem Setup}
\label{sec:problem_setup}

Let $\sX$ be the set of all possible data points. In our application, this includes both images and text. Contrastive learning assumes access to similar (positive) data pairs $(x,x^+)$ and $N$ i.i.d. negative samples $\pc{x_n^-}$ that are presumably unrelated to $x$. We use set $\sC$ of discrete latent classes to formalize the notion of semantic similarity, e.g., similar data $(x,x^+)$ have the same latent class. Let $\rho$ be the distribution over the latent class set $\sC$. We use $\sD_c$ to denote the probability distribution over $\sX$ that captures the likelihood of a data point $x$ belonging to a class $c\in\sC$. The data distribution $\sD$ is simply the marginal distribution, i.e., $\sD(x) \triangleq \sum_{c\in\sC} \rho(c) \sD_c(x)$. For convenience, we use $c_x$ to denote the latent class of $x\in\sX$. In practice, the class label $c_x$ is unknown.

The positive data pair $(x,x^+)\sim \sD_{\text{sim}}$ belongs to the same class by construction. This is achieved for example by (i) applying class-preserving data augmentation to generate an image $x^+$ from another image $x$ ($\sD_{\text{sim}}$ is a distribution of image and its data augmentations) or (ii) treating $x^+$ as the image associated with text $x$ ($\sD_{\text{sim}}$ is a distribution of image-text pairs). We define $\sD_{\text{sim}}(x,x^+) \triangleq \sD(x)\sD_{c_x}(x^+)$ to capture the constraint that $x^+$ is generated from the same latent class $c_x$ as $x$.

Negative samples $\pc{x_n^-}$ should be drawn from semantically dissimilar (w.r.t. $x$) data 
\begin{align}
    \sE_{c_x}(x')
        \triangleq p(x' \mid c\neq c_x)
        = \sum_{c\neq c_x} \frac{\rho(c)}{1-\rho(c_x)} \sD_c(x'),
\end{align}
which is infeasible since we do not have access to class labels $c_x$. In practice, negative samples $\pc{x_n^-}$ are typically drawn from the marginal $\sD$. Doing so introduces false negative samples. Specifically, we observe that sampling from $\sD$ is equivalent to sampling from a mixture of semantically similar data $\sD_{c_x}$ and semantically dissimilar data $\sE_{c_x}$, i.e., 
\begin{align}
    \sD(x') = \sum_{c\in\sC} \rho(c) \sD_c(x') =  \rho(c_x) \sD_{c_x}(x') + (1-\rho(c_x)) \sE_{c_x}(x').
    \label{eq:data_marginal_decompose_to_mixture}
\end{align}
With probability $\rho(c_x)$, $x_n^-\sim\sD$ is a false negative that is drawn from the component $\sD_{c_x}$. 

In many computer vision data sets, the number of classes is large and the false negative rate $\rho(c_x)$ is likely small for any data point $x\in\sX$. Thus, using $\sD$ instead of $\sE_{c_x}$ to sample negative examples is reasonable. Moreover, it is natural to assume that all classes are equally (un)likely. Unfortunately, neither assumption is true in the healthcare setting. In clinical image data sets, the probability of encountering common pathologies in a randomly chosen image is not negligible. Furthermore, common pathologies appear much more frequently than rare ones. Debiased contrastive learning~\citep{chuangDebiasedContrastiveLearning2020} addresses the former problem by modifying the optimization function to explicitly account for the chance of false negatives uniformly across data points.

\subsection{Debiased Contrastive Learning} 
\label{sec:dcl}

Contrastive learning aims to find a good representation function $f:\sX\to\sS^{d-1}(\gamma)$ that encodes data in $\sX$ to a hypersphere of radius $\gamma$ where $\sS^{d-1}(\gamma) \triangleq \pc{x\in\R^D \mid \norm{x}_2 = \gamma}$. Specifically, the contrastive learning objective forces representations of positive pairs to be closer than those of negative pairs~\citep{oordRepresentationLearningContrastive2018} by minimizing
\begin{align}
    \E_{ 
        (x,x^+)   \sim  \sD_{\text{sim}},
        \{x_n^-\} \sim  \sD^N 
    } \pb{
            -\log \frac{
                e^{s(x,x^+)}
            }{
                e^{s(x,x^+)} + \sum_{n=1}^N e^{s(x,x_n^-)}
            }
    },
    \label{eq:loss_contrasive_vanilla}
\end{align}
where the bounded function $s\colon\sX\times\sX\to\R$ measures the similarity (e.g., a dot-product) between learned representations captured by the encoder function $f$. Negative samples are uniformly drawn from the data $\sD$, which poses a risk of sampling false negatives.

Re-arranging Equation~\ref{eq:data_marginal_decompose_to_mixture}, true negative sample distribution $\sE_{c_x}$ can be expressed in terms of $\sD$ and $\sD_{c_x}$ from which we can readily sample:
\begin{align}
    \sE_{c_x}(x')
        = \frac{1}{1-\rho(c_x)} \sD(x') - \frac{\rho(c_x)}{1-\rho(c_x)} \sD_{c_x}(x').
\end{align}
The asymptotic debiased contrastive learning objective~\citep{chuangDebiasedContrastiveLearning2020} considers the case where the number $N$ of negative samples drawn from $\sE_{c_x}$ goes to infinity:
\begin{align}
    \E_{(x,x^+)\sim\sD_{\text{sim}}} \pb{
        -\log \frac{
            e^{s(x,x^+)}
        }{
            e^{s(x,x^+)} + N\,\E_{x^-\sim\sE_{c_x}}\pb{ e^{s(x,x^-)} }
        }
    }.
\label{eq:loss_dcl_asymptotic}
\end{align}
Given $N$ samples $\pc{u_n}$ from $\sD$ and $M$ samples $\pc{v_m}$ from $\sD_{c_x}$, the expected value in the denominator of Equation~\ref{eq:loss_dcl_asymptotic} can be estimated with 
\begin{align}
    g(x, \pc{u_n}, \pc{v_m};\eta) 
        &\triangleq 
            \frac{1}{1-\eta } 
            \left[\frac{1}{N} \sum_{n=1}^N e^{s(x,u_n)}\right]
            - 
            \frac{\eta}{1-\eta } 
            \left[\frac{1}{M}  \sum_{m=1}^M e^{(x,v_m)} \right].
\label{eq:expectation-est} 
\end{align}
When the class distribution is uniform and $\rho(c_x)=\eta$ for all $x\in\sX$, the estimator is consistent, i.e., $g(x, \pc{u_n}, \pc{v_m};\eta) \overset{p}{\to} \E_{x^-\sim\sE_{c_x}}[ e^{s(x,x^-)} ]$ as $N, M\to\infty$. In debiased contrastive learning \citep{chuangDebiasedContrastiveLearning2020}, $\eta$ is treated as a hyperparameter and is assumed to be constant. To avoid numerical issues, the estimator $g(x, \pc{u_n}, \pc{v_m}; \eta)$ is lower bounded by its theoretical minimum $e^{-\gamma^2}$~\citep{chuangDebiasedContrastiveLearning2020}. The resulting debiased contrastive loss
\begin{align}
    \sL
    \triangleq 
        \E_{ (x,x^+)\sim \sD_{\text{sim}}, \{u_n\}\sim \sD^N, \{v_m\}\sim \sD_{c_x}^M }
        \pb{
        -\log \frac{
            e^{s(x,x^+)}
        }{
            e^{s(x,x^+)} + Ng(x, \pc{u_n}, \pc{v_m};\eta)
        }
    }
    \label{eq:loss_dcl}
\end{align}
reweights the positive and negative terms in the denominator.

\subsection{Sample-specific Class Probability Function $\eta(\cdot)$}
\label{sec:sample_specific_probability_function}

The formulation above uses a single hyperparameter~$\eta$ to correct for false negatives uniformly for every data point. We propose to employ an estimate of the class probability $\rho(c_x)$ with a sample-specific class probability function $\eta(\cdot)$, by replacing $\eta$ with $\eta(\cdot)$ in Equation~\ref{eq:expectation-est}:
\begin{align}
    g(x, \pc{u_n}, \pc{v_m};\eta) 
        \triangleq 
            \frac{1}{1-\eta(x) } 
            \left[\frac{1}{N} \sum_{n=1}^N e^{s(x,u_n)}\right]
            - 
            \frac{\eta(x)}{1-\eta(x) } 
            \left[\frac{1}{M}  \sum_{m=1}^M e^{(x,v_m)} \right]
\label{eq:expectation-est-sample-specific}      
\end{align}
and use the objective function defined in Equation~\ref{eq:loss_dcl} with $g(\cdot)$ in Equation~\ref{eq:expectation-est-sample-specific}. In practice, $\eta(\cdot)$ may be an imperfect estimator for $\rho(c_x)$. The following proposition informs us how the quality of the estimator $g(\cdot)$ affects the approximation error.

\begin{proposition}
    \label{prop:dcl_estimation_error}
    Let $f$ and $\eta$ be arbitrary functions, $N$ and $M$ be finite. Then, 
    \begin{align}
        \label{eq:dcl_estimation_error_line1}
        \left|
            \sL - \overline{\sL}
        \right|
        &\leq
            \frac{3e^2\sqrt{\pi/2}}{\sqrt{N}} \, 
             \E_{x\sim\sD}\pb{\frac{1}{1-\rho(c_x)}}
            + \frac{3e^2\sqrt{\pi/2}}{\sqrt{M}} \, \E_{x\sim\sD}\pb{ \frac{\rho(c_x)}{1-\rho(c_x)} } \\
        &\quad+ 3e^2 \E_{x\sim\sD}\pb{ \left| \frac{1}{1-\eta(x)} - \frac{1}{1-\rho(c_x)} \right| },
        \label{eq:dcl_estimation_error_line2} \\ 
        &\text{where}\quad
        \overline{\sL}
            \triangleq \E_{(x,x^+)\sim\sD_{\text{sim}}} \pb{
            -\log \frac{
                e^{s(x,x^+)}
            }{
                e^{s(x,x^+)} + N\,\E_{x^-\sim\sE_{c_x}}\pb{ e^{s(x,x^-)} }
            }
        }.
    \end{align}
\end{proposition}
Proof is provided in Appendix~\ref{sec:dcl_estimation_error_proof}. 

Proposition~\ref{prop:dcl_estimation_error} bounds the approximation error due to (i) finite sample approximation in Equation~\ref{eq:expectation-est} (first two terms $\sO(\tfrac{1}{\sqrt{N}} + \tfrac{1}{\sqrt{M}})$) and (ii) misspecification of the sample-specific class probability $\eta$ (last term). When $\rho$ is uniform and $\eta(x) = \rho(c_x)$ is a constant, Proposition~\ref{prop:dcl_estimation_error} reduces to the result in \citet{chuangDebiasedContrastiveLearning2020}, up to constant factors. Assuming access to $c_x$ and using the correct class distribution $\rho$, i.e., $\eta(x)=\rho(c_x)$, yields a tighter error bound since the last term in the approximation error bound (Equation~\ref{eq:dcl_estimation_error_line2}) vanishes. Therefore, we aim to use sample-specific class probability function $\eta(x)$ that closely matches the true sample-specific class probability $\rho(c_x)$. 

In Appendix~\ref{appendix:generalization_bounds}, we show classification generalization bounds on representations trained using debiased contrastive loss with the sample-specific probability function $\eta(\cdot)$ in Equation~\ref{eq:expectation-est-sample-specific} for arbitrary class distribution.

\subsection{Language Model Estimate of Class Probability}
\label{sec:lm_estimate_of_eta}

We employ the likelihood of the text $p_{\text{LM}}(x)$ provided by a language model (LM) for text~$x$ to construct the estimate $\eta(x)$ of the class probability $\rho(c_x)$. Language models naturally provide estimates of token sequence probabilities. The estimates get better if the language model is fine-tuned on the data $\sD$. We assume a log-linear relationship between text and class probabilities, i.e., $\eta_{\text{LM}}(x) = a \cdot p_{\text{LM}}(x)^k$ where $a$ and $k$ are hyperparameters. Figure~\ref{fig:pseudocode} provides pseudocode for our proposed method, which requires a slight modification of the contrastive learning algorithm and incurs minimal computational overhead.

\section{Experiments}

We evaluate the advantages of estimating sample-specific class probability for contrastive learning in two different experiments.

\subsection{CIFAR10}

In this experiment, we evaluate image-only representations on data sets with a controlled class distribution. We do not attempt to estimate the class probabilities since no text information is available and would not be helpful since we control the frequency with which each class is included in the data set. Instead, we use the true class distribution during training and examine the effect of class distribution of the data set on the resulting representations.

\paragraph{Data}
For each $r\in\{0.05,0.1,0.25,0.5,0.75,0.9\}$, we generate a CIFAR10-$r$ subset of the CIFAR10 data set \citep{krizhevskyLearningMultipleLayers2009} as follows. The original data set includes 6,000 images for each class. For each of $5$ selected classes (dog, frog, horse, ship, truck), we randomly draw and include $r$ fraction of the images. We keep all images for each of the remaining $5$ classes. The CIFAR10-$r$ data set has a class probability of $\eta_{\text{Low}} = 0.2 r/(1+r)$ for each of the 5 selected classes and $\eta_{\text{High}} = 0.2 /(1+r)$ for the remaining classes. Larger values of $r$ lead to a more uniform class distribution. 

\paragraph{Representation Learning}
Similar to \citet{chuangDebiasedContrastiveLearning2020}, we use SimCLR \citep{chenIncrementalFalseNegative2022} for contrastive learning of image representations. We employ ResNet-18 \citep{heDeepResidualLearning2016} as the image encoder $f$ and dot-product as the similarity function $s$. The encoder is followed by a $2$-layer perceptron to create a $128$-dimensional embedding. For a reference image $x$, we employ data augmentation to generate the positive sample $x^+$. We set $v_1=x^+$ ($M=1$) and draw negative samples $\pc{u_n}$ and $\pc{x_n^-}$ randomly from the data set ($N=254$). We set $\gamma=\sqrt{2}$. We use the Adam optimizer \citep{kingmaAdamMethodStochastic2014} with a learning rate of $10^{-3}$ and weight decay of $10^{-6}$. Each model is trained for $300$ epochs. 

\paragraph{Contrastive Loss Variants}
We learn image representations using four different types of contrastive objectives: (i) baseline contrastive learning without correction for false negatives (CL), (ii) debiased contrastive learning that uses a constant $\eta_{\text{Low}}$ providing the true class probabilities for the $5$ subsampled classes (DCL-$\eta_{\text{Low}}$) and misspecified class probabilities for the remaining classes (iii) debiased contrastive learning that uses a constant $\eta_{\text{High}}$ providing misspecified class probabilities for the subsampled classes and the true class probability for the remaining 5 classes (DCL-$\eta_{\text{High}}$), and (iv) debiased contrastive learning with the true sample-specific class probability function for all samples (DCL-$\eta_{\text{True}}$). 

\paragraph{Image Classification}
We evaluate the quality of image representations with linear classification. We train a linear classifier with the cross-entropy loss and a fixed pretrained image encoder. We use the Adam optimizer with a learning rate of $10^{-3}$ and weight decay of $10^{-6}$. Each model is trained for $100$ epochs. We report classification accuracy with varying number of annotated examples available for training the classifier.

\begin{figure}
    \centering  
    \ifhighresfig
        \includegraphics[width=.8\textwidth]{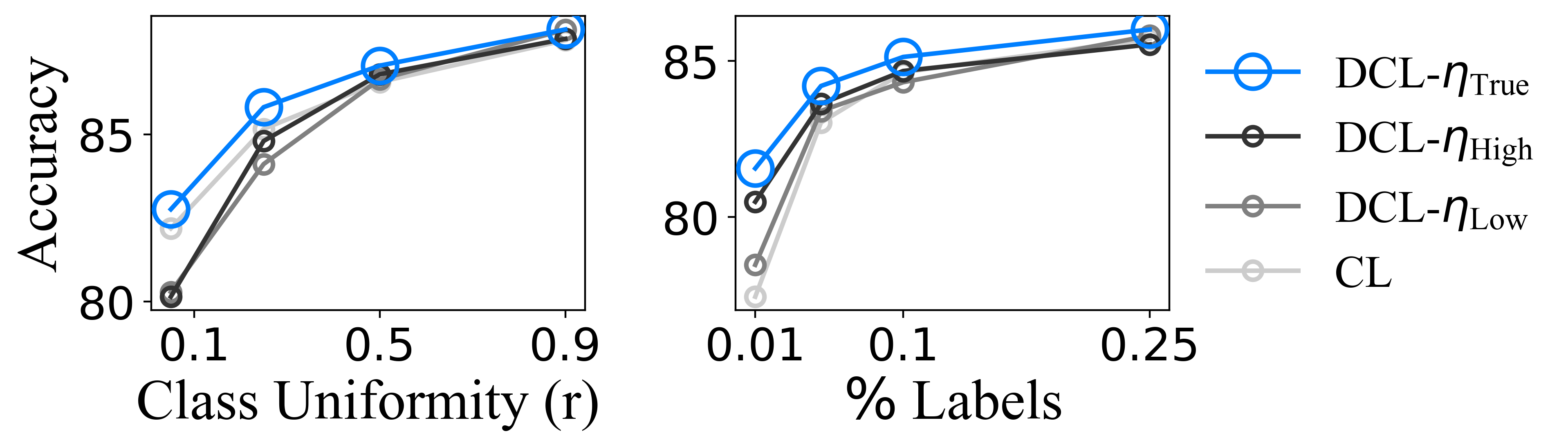} 
    \else
        \includegraphics[width=.8\textwidth]{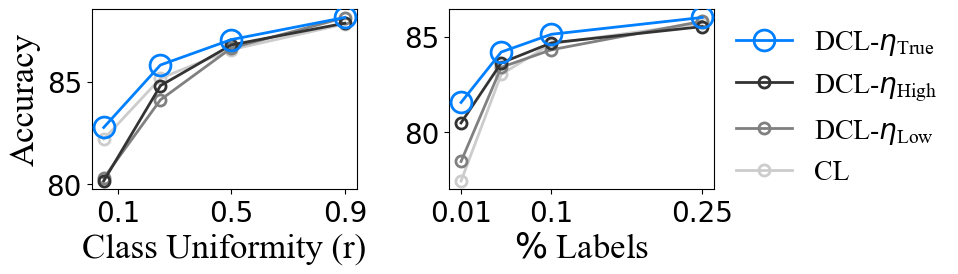} 
    \fi
    \caption{Evaluation using downstream classification task for CIFAR10 experiments. Left plot reports classification accuracy as a function of class distribution uniformity $r$. Right plot shows classification accuracy as a function of the fraction of images used for training the classifier on the CIFAR10-$0.5$ data set. Debiased contrastive learning with true class probabilities (DCL-$\eta_{\text{True}}$) consistently outperforms baseline contrastive learning (CL) and debiased contrastive learning with misspecified class probabilities (DCL-$\eta_{\text{Low}}$ and DCL-$\eta_{\text{High}}$). The effect is more pronounced when fewer labels are used for training the classifier.}
    \label{fig:cifar10_quantitative}  
\end{figure}

\paragraph{Results}
Figure~\ref{fig:cifar10_quantitative} reports the effect of the choice of sample-specific class probability function $\eta$ on classification accuracy. When the class distribution is nonuniform (i.e., smaller $r$), we observe that DCL-$\eta_{\text{True}}$ consistently outperforms CL, DCL-$\eta_{\text{Low}}$, and DCL-$\eta_{\text{High}}$. When the class distribution is close to uniform (i.e., $r$ is close to 1), all contrastive loss variants result in similar accuracy. 
Moreover, the gain from using the true sample-specific class probability is more pronounced when the classifier is trained with fewer labels. Figure~\ref{fig:cifar10_qualitative} provides t-SNE \citep{maatenVisualizingDataUsing2008} visualizations of the representations learned by contrastive and debiased contrastive objectives on CIFAR10-$0.1$. Using the true sample-specific class probability function $\eta_{\text{True}}$ leads to better class separation, especially for the subsampled classes.

\subsection{MIMIC-CXR}

We learn image and text encoders for frontal chest X-ray images and associated radiology reports respectively and evaluate the resulting representations in a set of downstream tasks.

\paragraph{Data} We use a subset of 234,073 frontal chest X-ray images and reports from MIMIC-CXR \citep{johnsonMIMICCXRDeidentifiedPublicly2019a}. 
We normalize the images and resize them to 512x512 resolution. We apply random image augmentations, i.e., 480x480 random crops, brightness and contrast variations. We use PySBD \citep{sadvilkarPySBDPragmaticSentence2020} for sentence tokenization. 
In all experiments, the data used for representation learning and downstream tasks are disjoint.

\paragraph{Representation Learning}

\begin{figure}[ht]
    \centering 
    \ifhighresfig
        \includegraphics[width=\textwidth]{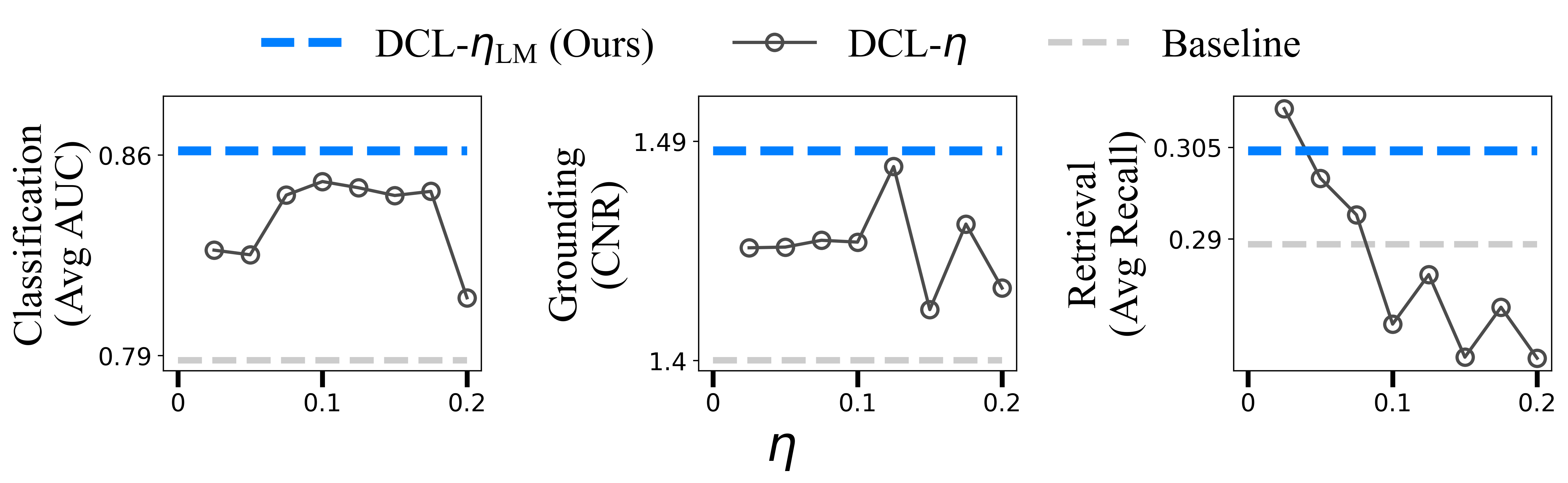}
    \else
        \includegraphics[width=\textwidth]{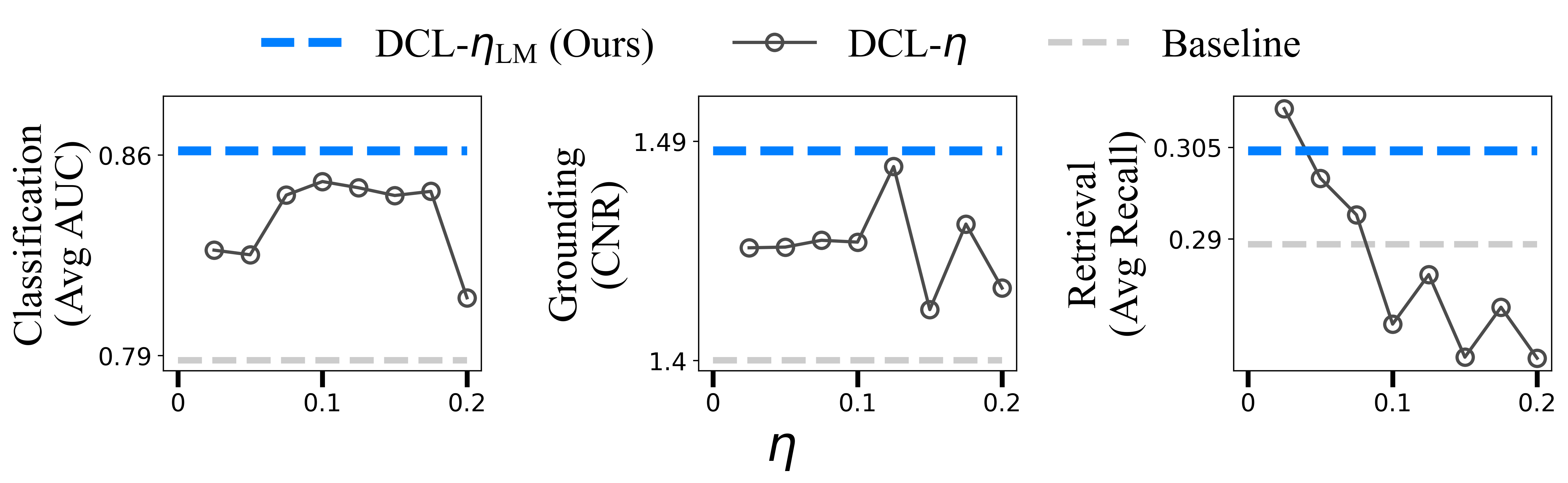}
    \fi
    \caption{
        Our method (DCL-$\eta_{\text{LM}}$) outperforms debiased contrastive learning that applies a fixed amount of correction $\eta$ to all samples (DCL-$\eta$) and LSE$+$NL that does not correct for the false negatives (Baseline). For each value of $\eta$, an image-text model is trained and subsequently evaluated in all three downstream tasks. Our method achieves consistently better performance than alternative approaches. 
        }
    \label{fig:mlhc_dcl_fixtaup_tradeoff}
\end{figure}

We employ ResNet-18 \citep{heDeepResidualLearning2016} as the image encoder and CXR-BERT \citep{boeckingMakingMostText2022} as the sentence encoder. Each encoder is followed by a linear mapping to a $128$-dimension embedding space. We use LSE+NL \citep{wangUsingMultipleInstance2023} as the similarity function $s$. Given a reference text $x$, we assign $x^+$ to be the associated image. We set $v_1=x^+$ to avoid needing additional samples ($M=1$) and treat all unpaired images in a batch as $\pc{x_n^-}$ and $\pc{u_n}$. After a grid search, we set $a=0.2$ and $k=0.35$. We precompute $p_{\text{LM}}(x)$ for all sentences $x$ in the data set using CXR-BERT. Masked language models (e.g., CXR-BERT) cannot estimate sentence probabilities via the chain rule. Instead, we use pseudo-log-likelihood \citep{salazarMaskedLanguageModel2020} that scores a sentence~$x$ by adding the predicted log probabilities of every masked token as $\log p_{\text{LM}}(x)$. We use the AdamW optimizer \citep{loshchilovDecoupledWeightDecay2019} and decay the initial learning rate of $10^{-5}$ using a cosine schedule with 2k warmup steps. We initialize $\gamma$ to $\sqrt{14}$ and optimize this hyperparameter alongside the encoder parameters. We employ a batch size of 64.

\paragraph{Baseline Methods}
We compare our method (DCL-$\eta_{\text{LM}}$) with strong baselines BioViL \citep{boeckingMakingMostText2022} and LSE$+$NL \citep{wangUsingMultipleInstance2023} developed specifically for medical vision-language tasks. BioViL is an image-text model trained using symmetric contrastive learning and masked language modeling objective. LSE$+$NL is an image-text model that uses log-sum-exp and non-local aggregators for the similarity function $s$. Neither model corrects for the false negatives explicitly.

We also include in the evaluation several methods that explicitly identify false negatives based on the similarity measure between the positive sample $x^+$ and negative samples $x_n^-$, such as the intersection of their CheX5 labels (CheX5 Labels) or similarity of their text embeddings (Text Sim.). Negative samples that are too similar (i.e., above some threshold) are removed \citep{khoslaSupervisedContrastiveLearning2020,huynhBoostingContrastiveSelfSupervised2022,zhangMedicalSymptomDetection2022,zhouDebiasedContrastiveLearning2022}. Alternatively, we can reweight the negative samples or even reduce the set of negative samples based on their similarity score \citep{robinsonContrastiveLearningHard2021,wangImprovingMolecularContrastive2022}. We perform grid searches to select hyperparameters (e.g., the similarity threshold or the resampling size) and select the best model for each setup.

\begin{table}[ht]
    \centering
    \begin{tabular}{L{0.320\linewidth}R{0.075\linewidth}L{0.075\linewidth}R{0.075\linewidth}L{0.075\linewidth}R{0.075\linewidth}L{0.075\linewidth}p{0pt}}
    \toprule
    Method & \multicolumn{2}{c}{Classification} & \multicolumn{2}{c}{Grounding} & \multicolumn{2}{c}{Retrieval} &  \\
     & \scriptsize AUC$\uparrow$ & \scriptsize ACC$\uparrow$ & \scriptsize CNR$\uparrow$ & \scriptsize mIoU$\uparrow$ & \scriptsize Recall$\uparrow$ & \scriptsize MedR$\downarrow$ &  \\
    \midrule
    \parbox{1.1in}{No Correction} BioViL & 0.78 & 0.62 & 1.14 & 0.17 & 0.25 & 148 &  \\
    \parbox{1.1in}{No Correction} LSE$+$NL & 0.79 & 0.65 & 1.40 & 0.19 & 0.29 & 115 &  \\    \midrule
    \parbox{.7in}{Remove} by CheX5 Labels & \textBF{0.86} & 0.71 & 1.37 & 0.19 & 0.29 & 113 &  \\
    \parbox{.7in}{Resample} by Text Sim. & 0.82 & 0.71 & 1.37 & 0.19 & \textBF{0.30} & 111 &  \\
    \parbox{.7in}{Remove} by Text Sim. & 0.84 & \textBF{0.72} & 1.39 & 0.19 & \textBF{0.30} & 112 &  \\
    \parbox{.7in}{Reweight} by Text Sim. & 0.84 & 0.69 & 1.40 & 0.19 & 0.29 & 113 &  \\
    \midrule
    DCL-$\eta$ w/ $\eta=0.05$ & 0.80 & \textBF{0.72} & 1.46 & 0.19 & \textBF{0.30} & \textBF{104} &  \\
    DCL-$\eta$ w/ $\eta=0.1$ & 0.85 & \textBF{0.72} & 1.45 & 0.19 & 0.29 & 111 &  \\
    DCL-$\eta_{\text{LM}}$ (Ours) & \textBF{0.86} & \textBF{0.72} & \textBF{1.49} & \textBF{0.20} & \textBF{0.30} & \textBF{104} &  \\
    \bottomrule
    \end{tabular}
    \caption{Zero-shot performance of the learned representations on downstream image classification, visual grounding, and cross-modal retrieval tasks. Debiased contrastive learning with sample-specific class probability $\eta_{\text{LM}}$ (DCL-$\eta_{\text{LM}}$) outperforms state-of-the-art baseline methods BioViL and LSE$+$NL (no false negative correction) and alternative approaches to false negative correction. While methods that correct for false negatives by removing, resampling, or reweighting are effective in improving image classification results for commonly occurring classes, they do not yield comparable improvements in visual grounding or retrieval results.
    } 
    \label{tab:cmp_fn_methods}
\end{table}

In addition, we include the original debiased contrastive learning objective that uses a constant $\eta$ \citep{chuangDebiasedContrastiveLearning2020}.

\paragraph{Downstream Tasks}

We assess the zero-shot image classification performance of $5$ CheXpert labels (Cardiomegaly, Edema, Pleural Effusion, Pneumonia, Pneumothorax) on the MIMIC-CXR data set \citep{johnsonMIMICCXRDeidentifiedPublicly2019a} that we refer to as CheX5. There is roughly 1k images for each binary classification task. We first tokenize and encode class-specific text prompts (e.g., ``No signs of pneumonia.'' or ``Findings suggesting pneumonia.''). Table~\ref{tab:zs_cls_prompts} provides the prompts for each category. For every image, we assign a binary label that corresponds to the prompt with the higher image-sentence score. We report classification accuracy (ACC) and area under the curve (AUC).

We evaluate visual grounding performance using the MS-CXR region-sentence annotations \citep{boeckingMakingMostText2022}. This data set consists of 1,448 bounding boxes over 1,162 images, where each bounding box is associated with a sentence that describes its dominant radiological feature. We compute region-sentence scores to quantify how well the sentence is localized in the image. We report a measure of discrepancy between region-sentence scores inside and outside the bounding box, i.e., contrast-to-noise ratio (CNR) \citep{boeckingMakingMostText2022}, and how well the thresholded region-sentence scores overlap with the bounding box on average, i.e., mean intersection over union (mIoU). We use thresholds that span $[-1,1]$ in $0.05$ increments to compute the mIoU.

We evaluate cross-modal retrieval performance using the MS-CXR data set. To evaluate retrieval, we compute the bounding box features from the region features with RoIAlign \citep{heMaskRCNN2017}. We compute box-sentence scores and sort them to retrieve items in one modality given a query from the other modality. The correctly retrieved item is the one that is paired with the query item. We compute the fraction of times the correct item was found in the top $K$ results (R@K), the median rank of the correct item in the ranked list (MedR), and the average recall over $K=10,50,100$ and over both the image-to-text and the text-to-image retrieval direction (Recall).

\paragraph{Results}

Figure~\ref{fig:mlhc_dcl_fixtaup_tradeoff} illustrates performance trade-off for debiased contrastive learning that uses a constant class probability function $\eta$. Increasing the value of $\eta$ improves image classification but harms cross-modal retrieval performance. Conversely, decreasing the value of $\eta$ enhances cross-modal retrieval performance but harms image classification performance. DCL-$\eta_{\text{LM}}$  provides an overall superior solution on all three tasks.

Table~\ref{tab:cmp_fn_methods} reports the performance of DCL-$\eta_{\text{LM}}$ and competing image-text models. Training LSE$+$NL with DCL-$\eta_{\text{LM}}$ significantly improves its performance on these tasks compared to BioViL and LSE$+$NL. Using a sample-specific class probability function $\eta_{\text{LM}}$ is more effective than using a constant function $\eta=0.05$ and $0.1$. The methods that identify and remove, resample, or reweight the false negatives improve image classification performance but do not offer any  gain for visual grounding and minimal improvement for cross-modal retrieval. Tables~\ref{tab:sota_zs_cls}, \ref{tab:sota_grounding}, and \ref{tab:sota_retrieval} in Appendix~\ref{appendix:additional_results} provide additional statistics for each task.

\section{Discussion}

We introduced a novel sample-specific approach that corrects the effect of false negative samples on the contrastive objective. Consistent with prior work~\citep{chuangDebiasedContrastiveLearning2020}, this approach offers empirical advantages when the learned representation is used in downstream tasks. In addition, the sample-specific approach to correcting false negatives improves the performance over the original variant of debiased contrastive learning that applied the same correction for all data points.

Our experiments also demonstrate that reducing false negatives for tasks with varying levels of granularity are nuanced. In particular, the methods that attempt to remove, resample, or reweight likely false negative examples can improve image classification performance but offer minimal or no improvement for visual grounding and cross-modality retrieval tasks. We hypothesize that this is because the fine-grain classes that must be also handled for the latter two tasks make identifying false negative samples more error-prone. Similarly, the performance of representations learned via debiased contrastive learning with a fixed correction factor is sensitive to the value of the assumed class probability and varies substantially across the range of possible values. Moreover, the optimal choice of the correction factor varies across downstream tasks, making it challenging to train universally useful representations. To the best of our knowledge, our work is the first to identify the differences in performance of vision-language tasks due to choices in correcting for the effect of false negatives during representation learning. The proposed sample-specific approach applies adaptive correction and produces representations that achieve superior performance.

Using all CheX5 labels to remove false negatives consistently improves the performance of classifying the CheX5 classes. However, doing so has minimal impact on visual grounding and cross-modal retrieval tasks that require the ability to distinguish rare classes beyond those specified in CheX5. Methods that require a clear definition of the latent classes \citep{khoslaSupervisedContrastiveLearning2020,dwibediLittleHelpMy2021} are not easily adaptable to situations where there are many classes or when the classes are difficult to define concretely. In contrast, methods that implicitly define the latent classes via clustering \citep{chenIncrementalFalseNegative2022}, or do not assume access to the latent classes (\citet{chuangDebiasedContrastiveLearning2020} or our approach) show promise in improving model performance on vision-language tasks.

In our experiments, visual grounding and cross-modal retrieval tasks use the MS-CXR data set that contains infrequently occurring sentences. Thus, one would assume both tasks would benefit more from using smaller values of the assumed uniform class probability $\eta$ than image classification of commonly occurring pathologies. However, the observed performance improvement for grounding is less noticeable than for retrieval. One reason is that the overlap-based metrics (e.g., IoU) used to evaluate grounding may be less sensitive to the variations in text embeddings representing different but related classes. In contrast, to achieve a high value of ranking-based metrics (e.g., recall) requires the model to rank the correct sentence higher than closely related alternatives. This unexpected result suggests that in addition to the class distribution, target performance metrics that capture how the model will be used in the clinical application are also important when developing representation learning approaches.

Our work emphasizes the need to consider the unique characteristics of medical data that pose challenges for methods designed for simpler scenarios. When working with paired image-text data, we observe that false negatives occur at an uneven rate, making methods designed to address false negatives assuming a uniform class distribution less effective. Furthermore, our work shows the potential of using language models to develop more effective algorithms. Language models are versatile, i.e., capable of processing noisy language inputs, and can serve as useful priors.

\subsection*{Limitations}

While the theory presented in this paper is informative, it has limitations. For example, the generalization bounds provided in Appendix~\ref{appendix:generalization_bounds} assume a specific similarity function (e.g., dot product) that may not represent the various similarity functions used in practice.  Moreover, the generalization bounds only apply to downstream classification tasks while we also evaluate on visual grounding and cross-modal retrieval tasks. While these tasks can be interpreted as some form of nearest neighbor classification, the theoretical results do not trivially extend to these scenarios and more analysis is needed to provide similar generalization guarantees for these tasks.

We use a language model to score text as a proxy for the sample-specific class probability. However, it is unclear how to apply this strategy to data that does not include associated text. While we can estimate data density with certain types of models  (e.g., flow-based or autoregressive models), it is not well understood whether these estimates correlate well with the underlying class probabilities. Natural language data is unique in that it is created by humans, capturing important variations in data that align well with the types of problems that users typically wish to solve.

Moreover, it is uncertain how well our assumption that class probability is log-linear with respect to text sequence probability holds in practice. Verifying this assumption requires defining a concrete set of latent classes and annotating text with the defined classes. However, the latent classes are difficult to define in practice, and we do not have access to the mapping from a data point to its latent class in most clinically important problems as having access to such mapping would eliminate the need for self-supervised learning.

\section{Conclusion} 

We present a debiased contrastive learning framework that accommodates arbitrary class distributions and mitigates the impact of false negatives. We offer theoretical and empirical evidence for using accurate sample-specific class probabilities. To this end, we introduce a specific debiased contrastive objective that employs a language model to estimate the sentence likelihood as a proxy for the class probability. When applied to paired image-text data, our method outperforms strong image-text models on image classification, visual grounding, and cross-modal retrieval tasks.

\acks{This work was supported by NIH NIBIB NAC P41EB015902, Philips, Wistron, and MIT Lincoln Laboratory. We thank Neel Dey and Nalini Singh for proofreading and providing feedbacks on paper drafts.}

\newpage
\newpage
\bibliography{all}

\newpage
\appendix

\section{Theoretical Results}

\renewcommand{\thelemma}{A.\arabic{lemma}}

In this section. We use use $g_0$ to denote the estimator in Equation~\ref{eq:expectation-est-sample-specific}
\begin{align}
    g_0(x, \pc{u_n}, \pc{v_m}) 
        \triangleq 
            \frac{1}{1-\eta(x)}\pa{
                \frac{1}{N} \sum_{n=1}^N e^{s(x,u_n)}
                - \eta(x) \frac{1}{M} \sum_{m=1}^M e^{(x,v_m)}
            }
\end{align}
and $g(x, \pc{u_n}, \pc{v_m}) = \max(g_0(x, \pc{u_n}, \pc{v_m}), e^{-\gamma^2})$ as the version of estimator lower bounded by its theoretical minimum.

\subsection{Proof for Proposition~\ref{prop:dcl_estimation_error}}
\label{sec:dcl_estimation_error_proof}
\begin{proof}

The goal is to show how well debiased contrasive loss $\sL$ approximates the asymptotic debiased contrastive loss $\overline{\sL}$. Without loss of generality, we assume $ \gamma=1$. The proof holds as long as $s$ is bounded. For simplicity, we assume $s(x,x')=f(x)^Tf(x')$, implying $s(x)\in[-1,1]$.

Let $\varepsilon>0$ and $x,x^+\in\sX$ be arbitrary. We are interested in the tail probability of the difference between the integrands of $\sL$ and $\overline{\sL}$, i.e., $\Prob(\triangle \geq \varepsilon)$ where 
\begin{align*}
    \triangle
        = \left|
            -\log \frac{e^{s(x,x^+)}}{e^{s(x,x^+)} + Ng(x, \pc{u_n}, \pc{v_m})} 
            +\log \frac{e^{s(x,x^+)}}{e^{s(x,x^+)} + N\E_{x^-\sim\sE_{c_x}}\pb{e^{s(x,x^-)}}}
        \right|.
\end{align*}
Note, $\triangle$ implicitly depends on $x,x^+$. $\pc{u_n},\pc{v_m}$ are random samples. Now simplify:
\begin{align*}
    \Prob\pa{
        \triangle \geq \varepsilon
    }
        &= \Prob\pa{
            \left|
                \log\pa{e^{s(x,x^+)} + Ng(x,\pc{u_n},\pc{v_m})} - 
                \log\pa{e^{s(x,x^+)} + N\E_{x^-\sim\sE_{c_x}}\pb{e^{s(x,x^-)}}}
            \right| \geq \varepsilon
        } \\ 
        &= \Prob\pa{
            \log\pa{e^{s(x,x^+)} + Ng(x,\pc{u_n},\pc{v_m})} - 
            \log\pa{e^{s(x,x^+)} + N\E_{x^-\sim\sD_{c_x}}\pb{e^{s(x,x^-)}}}
            \geq \varepsilon
        } \\ 
        &\quad+ \Prob\pa{
            -\log\pa{e^{s(x,x^+)} + Ng(x,\pc{u_n},\pc{v_m})} + 
            \log\pa{e^{s(x,x^+)} + N\E_{x^-\sim\sD_{c_x}}\pb{e^{s(x,x^-)}}}
            \geq \varepsilon
        }.
\end{align*}
The first term can be bounded, i.e.,
\begin{align*}
    &\Prob\pa{
            \log\pa{e^{s(x,x^+)} + Ng(x,\pc{u_n},\pc{v_m})} - 
            \log\pa{e^{s(x,x^+)} + N\E_{x^-\sim\sE_{c_x}}\pb{e^{s(x,x^-)}}}
            \geq \varepsilon
        } \\ 
    \quad&=\Prob\pa{
        \log \frac{ e^{s(x,x^+)} + Ng(x,\pc{u_n}.\pc{v_m}) }{ e^{s(x,x^+)} + N\E_{x^-\sim\sE_{c_x}}\pb{e^{s(x,x^-)}} } \geq \varepsilon
    } \\ 
    \quad&\leq\Prob\pa{
        \frac{ Ng(x,\pc{u_n}, \pc{v_m}) - N\E_{x^-\sim\sE_{c_x}}\pb{e^{s(x,x^-)}} }{ e^{s(x,x^+)} + N\E_{x^-\sim\sE_{c_x}}\pb{e^{s(x,x^-)}}  } \geq \varepsilon
    } \tag{ $\log x \leq x-1$ for $x>0$ } \\ 
    \quad&=\Prob\pa{
        g(x,\pc{u_n}, \pc{v_m}) - \E_{x^-\sim\sE_{c_x}}\pb{e^{s(x,x^-)}}
        \geq \varepsilon \pa{\frac{1}{N}e^{s(x,x^+)} + \E_{x^-\sim\sE_{c_x}}\pb{e^{s(x,x^-)}} }
    } \\ 
    \quad&\leq\Prob\pa{
        g(x,\pc{u_n}, \pc{v_m}) - \E_{x^-\sim\sE_{c_x}}\pb{e^{s(x,x^-)}}
        \geq \varepsilon e^{-1}
    } \tag{ $\frac{1}{N}e^{s(x,x^+)} + \E_{x^-\sim\sE_{c_x}}\pb{e^{s(x,x^-)}} \geq \frac{1}{N}e^{-1} + e^{-1} \geq e^{-1} $ }.
\end{align*}
Similarly, the second term is bounded in a similar manner. Since both term can be bounded,
\begin{align*}
    \Prob\pa{
        \triangle \geq \varepsilon
    }
        &\leq \Prob\pa{
            \left|
                g(x,\pc{u_n}, \pc{v_m}) - \E_{x^-\sim\sE_{c_x}}\pb{e^{s(x,x^-)}}
            \right|
            \geq \varepsilon e^{-1}
        } \\
        &\leq \Prob\pa{
            \left|
                g_0(x,\pc{u_n}, \pc{v_m}) - \E_{x^-\sim\sE_{c_x}}\pb{e^{s(x,x^-)}}
            \right|
            \geq \varepsilon e^{-1}
        }. \tag{for any $b\geq e^{-1}$, $|\max(a,e^{-1})-b|\leq |a-b|$ } 
\end{align*}
We decompose the term inside the absolute value into 3 terms:
\begin{align*}
    &g_0(x,\pc{u_n}, \pc{v_m}) - \E_{x^-\sim\sE_{c_x}}\pb{e^{s(x,x^-)}} \\ 
        \quad&=\frac{1}{1-\eta(x)} \pa{ \frac{1}{N}\sum_{n=1}^N e^{s(x,u_n)} - \eta(x) \frac{1}{M}\sum_{m=1}^M e^{s(x,v_m)} } \\ 
        \quad&\quad- \frac{1}{1-\rho(c_x)} \pa{
            \E_{x^-\sim\sD}\pb{e^{s(x,x^-)}} - \rho(c_x) \E_{x^-\sim\sD_{c_x}}\pb{e^{s(x,x^-)}}
        } \\ 
        \quad&=
            \frac{1}{1-\rho(c_x)}\pa{
                \frac{1}{N}\sum_{n=1}^N e^{s(x,u_n)} - 
                \E_{x^-\sim\sD}\pb{e^{(x,x^-)}}
            } +\frac{\rho(c_x)}{1-\rho(c_x)} \pa{
                \frac{1}{M}\sum_{m=1}^M e^{s(x,v_m)} - 
                \E_{x^-\sim\sD_{c_x}}\pb{e^{s(x,x^-)}}
            } \\ 
        \quad&\quad+
            \pa{\frac{1}{1-\eta(x)} - \frac{1}{1-\rho(c_x)}} \pa{
                \frac{1}{N} \sum_{n=1}^N e^{s(x,u_n)} - \frac{1}{M} \sum_{m=1}^M e^{s(x,v_m)}
            }. 
\end{align*}
Continue where we left off,
\begin{align*}
    \Prob\pa{
        \triangle \geq \varepsilon
    }
    &\leq \Prob\pa{
        \left|
            \frac{1}{1-\rho(c_x)}\pa{
                \frac{1}{N}\sum_{n=1}^N e^{s(x,u_n)} - 
                \E_{x^-\sim\sD}\pb{e^{s(x,x^-)}}
            } 
        \right|
        \geq \frac{\varepsilon e^{-1}}{3}
    } \\ 
    &\quad+\Prob\pa{
        \left|
            \frac{\rho(c_x)}{1-\rho(c_x)} \pa{
                \frac{1}{M}\sum_{m=1}^M e^{s(x,v_m)} - 
                \E_{x^-\sim\sD_{c_x}}\pb{e^{s(x,x^-)}}
            }
        \right|
        \geq \frac{\varepsilon e^{-1}}{3}
    } \\ 
    &\quad+\Prob\pa{
        \left|
            \pa{\frac{1}{1-\eta(x)} - \frac{1}{1-\rho(c_x)}} \pa{
                \frac{1}{N} \sum_{n=1}^N e^{s(x,u_n)} - \frac{1}{M} \sum_{m=1}^M e^{s(x,v_m)}
            }
        \right|
        \geq \frac{\varepsilon e^{-1}}{3}
    }.
\end{align*}
Hoeffding's Inequality states that given independent bounded random variable $Z_1,\cdots,Z_n$ where $Z_i\in[a,b]$ for all $i$, then $\Prob(|\frac{1}{n}\sum_{i=1}^n Z_i - \E\pb{Z_i}| \geq t) \leq 2\exp(-2nt^2 / (b-a)^2)$. Since $e^{-1}\leq e^{s(x,x')}\leq e$ for all $x'\in\sX$, the first two terms can be bounded by Hoeffding's Inequality
\begin{align*}
    \Prob\pa{
        \left|
            \frac{1}{1-\rho(c_x)}\pa{
                \frac{1}{N}\sum_{n=1}^N e^{s(x,u_n)} - 
                \E_{x^-\sim\sD}\pb{e^{s(x,x^-)}}
            } 
        \right|
        \geq \frac{\varepsilon e^{-1}}{3}
    }
        &\leq 2 \exp\pa{
            - \frac{2N\varepsilon^2}{9e^4} \pa{1-\rho(c_x)}^2
        }, \\ 
    \Prob\pa{
        \left|
            \frac{\rho(c_x)}{1-\rho(c_x)} \pa{
                \frac{1}{M}\sum_{m=1}^M e^{s(x,v_m)} - 
                \E_{x^-\sim\sD_{c_x}}\pb{e^{s(x,x^-)}}
            }
        \right|
        \geq \frac{\varepsilon e^{-1}}{3}
    } 
        &\leq 2\exp\pa{
            - \frac{2M\varepsilon^2}{9e^4} \pa{\frac{1-\rho(c_x)}{\rho(c_x)}}^2
        }.
\end{align*}
The last term can be upper bounded with an indicator
\begin{align*}
    &\Prob\pa{
        \left|
            \pa{\frac{1}{1-\eta(x)} - \frac{1}{1-\rho(c_x)}} \pa{
                \frac{1}{N} \sum_{n=1}^N e^{s(x,u_n)} - \frac{1}{M} \sum_{m=1}^M e^{s(x,v_m)}
            }
        \right|
        \geq \frac{\varepsilon e^{-1}}{3}
    }  \\
    \quad&\leq \Prob\pa{
        \left|
            \pa{\frac{1}{1-\eta(x)} - \frac{1}{1-\rho(c_x)}}
            \pa{e - e^{-1}}
        \right|
        \geq
        \frac{\varepsilon e^{-1}}{3}
    } \\
    \quad&\leq \mathbbm{1}\pc{
        \varepsilon \leq 3e^2 \left|
            \frac{1}{1-\eta(x)} - \frac{1}{1-\rho(c_x)}
        \right|
    }.
\end{align*}
Combine these 3 upper bounds together we have 
\begin{align*}
    \Prob\pa{
        \triangle \geq \varepsilon
    }
    &\leq 
    2 \exp\pa{
            - \frac{2N\varepsilon^2}{9e^4} \pa{1-\rho(c_x)}^2
    }
    + 2\exp\pa{
            - \frac{2M\varepsilon^2}{9e^4} \pa{\frac{1-\rho(c_x)}{\rho(c_x)}}^2
    } \\ 
    &\quad+ \mathbbm{1}\pc{
        \varepsilon \leq 3e^2 \left|
            \frac{1}{1-\eta(x)} - \frac{1}{1-\rho(c_x)}
        \right|
    }.
\end{align*}
Now we bound the approximation error using the tail probability we just derived. 
\begin{align*}
    \left| \sL - \overline{\sL} \right|
        &\leq \E_{(x,x^+)\sim\sD_{\text{sim}}}\pb{\triangle} 
            \tag{Jensen's Inequality} \\ 
        &= \E_{(x,x^+)\sim\sD_{\text{sim}}}\pb{
            \E\pb{ \triangle \mid x,x^+  }
        } \\ 
        &= \E_{(x,x^+)\sim\sD_{\text{sim}}}\pb{
            \int_{0}^{\infty} \Prob\pa{\triangle \geq \varepsilon \mid x,x^+}\, d\varepsilon
        }
            \tag{Write $\E$ using CDF} \\
        &\leq \E_{x\sim\sD} \pb{
            \int_0^{\infty} 2\exp\pa{
                - \frac{2N\varepsilon^2}{9e^4} \pa{1-\rho(c_x)}^2
            }\, d\varepsilon
        } \\
        &\quad+ \E_{x\sim\sD} \pb{
            \int_0^{\infty} 2\exp\pa{
                - \frac{2M\varepsilon^2}{9e^4} \pa{\frac{1-\rho(c_x)}{\rho(c_x)}}^2
            }\, d\varepsilon
        } \\ 
        &\quad+\E_{x\sim\sD} \pb{
            \int_0^{\infty} \mathbbm{1}\pc{
                \varepsilon \leq 3e^2 \left|
                    \frac{1}{1-\eta(x)} - \frac{1}{1-\rho(c_x)}
                \right|
            }\, d\varepsilon
        } \\ 
        &=
        3e^2 \sqrt{ \frac{\pi}{2N} } \E_{x\sim\sD}\pb{\frac{1}{1-\rho(c_x)}}
        + 3e^2 \sqrt{ \frac{\pi}{2M} } \E_{x\sim\sD}\pb{ \frac{\rho(c_x)}{1-\rho(c_x)} } \\
        &\quad+ 3e^2 \E_{x\sim\sD}\pb{ \left| \frac{1}{1-\eta(x)} - \frac{1}{1-\rho(c_x)} \right| }
            \tag{Use identity $\int_0^{\infty} e^{-cz^2}\, dz = \frac{1}{2} \sqrt{\frac{\pi}{c}}$}.
\end{align*}
\end{proof}

\subsection{Classification Generalization Bounds}
\label{appendix:generalization_bounds}

To show relationship between the different loss functions, we re-define the losses to show their dependence on the encoder $f$ and some additional parameters
\begin{align}
    \overline{\sL}(f, Q)
        &\triangleq \E_{(x,x^+)\sim\sD_{\text{sim}}} \pb{
            -\log \frac{
                e^{f(x)^Tf(x^+)}
            }{
                e^{f(x)^Tf(x^+)} + Q\E_{x^-\sim\sE_{c_x}}\pb{ e^{f(x)^Tf(x^-)} }
            }
        }, \\ 
    \sL(f, N, M)
        &\triangleq 
            \E_{ (x,x^+)\sim\sD_{\text{sim}}, \{u_n\}\sim\sD^N, \{v_m\}\sim\sD_{c_x}^M }
            \pb{
            -\log \frac{
                e^{f(x)^Tf(x^+)}
            }{
                e^{f(x)^Tf(x^+)} + Ng(x, \pc{u_n}, \pc{v_m})
            }
        }
    \label{eq:dcl_depends_on_fNM}
\end{align}
where the similarity function is dot product, i.e., $s(x,x') = f(x)^Tf(x')$. 

Following \citet{saunshiTheoreticalAnalysisContrastive2019,chuangDebiasedContrastiveLearning2020}, we characterize a representation $f$ on $K$-way classification task $\sT$ consisting of $K$ classes $\pc{c_1,\cdots,c_K}\subset \sC$. The supervised dataset is generated from $\sD_{\sT}(x,c) \propto \sD_c(x)\rho(c)$. Specifically, we fix the representations and train a linear classifier $h(x) \triangleq Wf(x)$ on task $\sT$ with softmax cross entropy loss $\sL_{\text{softmax}}(\sT,h)$. The supervised loss for classifier $h$ on task $\sT$ is 
\begin{align}
    \sL_{\text{sup}}(\sT, h)
        \triangleq \min_{W\in\R^{K\times D}} \sL_{\text{softmax}}(\sT,h).
\end{align}
We focus on the supervised loss of the mean classifier, where rows of $W$ are the means of the representations of inputs with label $c$, i.e., $W_{\mu} \triangleq \begin{bmatrix} \mu_1,\cdots,\mu_K \end{bmatrix}^T$ where $\mu_c = \E_{x\in\sD_c}\pb{f(x)}$. The supervised loss for the mean classifier is
\begin{align}
    \sL_{\text{sup}-\mu}(\sT, f)
        \triangleq \sL_{\text{sup}}(\sT, W_{\mu}f)
        = \E_{(x,c)\sim \sD_{\sT}} \pb{ - \log \frac{
            e^{f(x)^T\mu_c} 
        }{
            \sum_{c^-\in \sT} e^{f(x)^T\mu_{c^-}}
        } }.
\end{align}
We consider the average classification performance over the distribution of $K$-way multi-class classification tasks $p_{\sT}(\pc{c_i}_{k=1}^K) \propto \prod_{k=1}^K \rho(c_i) \cdot \mathbbm{1}\pb{c_i\neq c_j\,\forall i\neq j}$. The average supervised loss over tasks is 
\begin{align}
    \sL_{\text{sup}}(f) 
        \triangleq \E_{\sT\sim p_{\sT}} \sL_{\text{sup}}(\sT, f).
\end{align}

The following lemma bounds the supervised loss $\sL_{\text{sup}}$ with the asymptotic contrastive loss $\sL$.

\begin{lemma}
    \label{lem:sup_vs_dcl_loss_bounds}
    For any encoder $f$, whenever $N \geq \frac{1-\rho_{\text{min}}}{\rho_{\text{min}}}$ where $\rho_{\text{min}} \triangleq \min_{c\in\sC} \rho(c)$, we have 
    \begin{align}
        \sL_{\text{sup}}(f)
            \leq \sL_{\text{sup}-\mu}(f)
            \leq \overline{\sL}(f, N)
    \end{align}
\end{lemma}
\begin{proof}
    We first show that $N=\frac{1-\rho_{\text{min}}}{\rho_{\text{min}}}$ gives the smallest loss: 
    \begin{align}
        \overline{\sL}(f,N)
            &= \E_{(x,x^+)\sim\sD_{\text{sim}}} \pb{
                -\log \frac{
                    e^{f(x)^Tf(x^+)}
                }{
                    e^{f(x)^Tf(x^+)} + N\E_{x^-\sim \sE_{c_x}} \pb{
                        e^{f(x)^Tf(x^-)}
                    }
                }
            } \\ 
            &\geq \E_{(x,x^+)\sim\sD_{\text{sim}}} \pb{
                -\log \frac{
                    e^{f(x)^Tf(x^+)}
                }{
                    e^{f(x)^Tf(x^+)} + \frac{1-\rho_{\text{min}}}{\rho_{\text{min}}} \E_{x^-\sim \sE_{c_x}} \pb{
                        e^{f(x)^Tf(x^-)}
                    }
                }
            } \\ 
            &= \overline{\sL}\pa{
                f, \frac{1-\rho_{\text{min}}}{\rho_{\text{min}}} 
            }.
    \end{align}
    The task-specific class distribution is $\rho_{\sT}(\pc{c_k}) \propto \prod_{k=1}^K \rho(c_k)$. Now we show the asymptotic debiased loss upper bounds the supervised classification loss:
    \begin{align}
        &\overline{\sL}\pa{
            f, \frac{1-\rho_{\text{min}}}{\rho_{\text{min}}} 
        } \\
        &= \E_{(x,x^+)\sim\sD_{\text{sim}}} \pb{
                    -\log \frac{
                        e^{f(x)^Tf(x^+)}
                    }{
                        e^{f(x)^Tf(x^+)} + \frac{1-\rho_{\text{min}}}{\rho_{\text{min}}} \E_{x^-\sim \sE_{c_x}} \pb{
                            e^{f(x)^Tf(x^-)}
                        }
                    }
                } \\
        &= \E_{\sT\sim p_{\sT}, c\sim\rho_{\sT}, x,x^+\sim \sD_c^2} \pb{
                -\log\frac{
                    e^{f(x)^Tf(x^+)}
                }{
                    e^{f(x)^Tf(x^+)} + \frac{1-\rho_{\text{min}}}{\rho_{\text{min}}} \E_{\sT\sim p_{\sT}, c^-\sim \rho_{\sT}(\cdot\mid c^-\neq c), x^-\sim \sD_{c^-}} \pb{
                        e^{f(x)^Tf(x^-)}
                    }
                }
            } \\
        &\geq \E_{\sT\sim p_{\sT}, (x,c)\sim\sD_{\sT}} \pb{
            -\log\frac{
                e^{f(x)^T\mu_{c^+}}
            }{
                e^{f(x)^T\mu_{c^+}} +  \frac{1-\rho_{\text{min}}}{\rho_{\text{min}}} \E_{c^-\sim \rho_{\sT}(\cdot\mid c^-\neq c)} \pb{
                    e^{f(x)^T\mu_{c^-}}
                }
            }
        } \tag{Jensen's Inequality} \\
        &\geq \E_{\sT\sim p_{\sT}, (x,c)\sim\sD_{\sT}} \pb{
            -\log\frac{
                e^{f(x)^T\mu_{c^+}}
            }{
                e^{f(x)^T\mu_{c^+}} + 
                    \sum_{c^-\in \sT:c^-\neq c} e^{f(x)^T\mu_{c^+}}
            }
        } 
        \tag{$\frac{1-\rho_{\text{min}}}{\rho_{\text{min}}} \geq \frac{\sum_{c_k:c_k\neq c} \rho(c_k) }{\rho(c^-)}$ for any $c^-,c_1,\cdots,c_K\in\sC^{K+1}$}
        \\
        &= \sL_{\text{sup}-\mu}(f).
    \end{align}
    We have showed the second inequality in the Lemma. The first inequality follows from the definition of $\sL_{\text{sup}}$, i.e., $\sL_{\text{sup}}$ is the minimal loss over all weight matrices $W$ including that of the mean classifier $W_{\mu}$.
\end{proof}

We wish to derive a data dependent bound. We follow the proof strategy detailed in \citet{saunshiTheoreticalAnalysisContrastive2019}. We assume that the class-specific probability function depends on the representation, i.e., $\eta(x) = \kappa(f(x))$ for some $\kappa:\R^D\to[\eta_{\text{min}}, \eta_{\text{max}}]$. This assumption is introduced to use theoretical results from prior work; We leave generalization to arbitrary $\eta$ to future work. First, we express the debiased contrastive loss in an alternative form:
\begin{align}
    \label{eq:dcl_reexpress_using_loss}
    &\sL(f, N, M)
        = \E_{(x,x^+\sim\sD_{\text{sim}}), \pc{u_n} \sim \sD^N, \pc{v_m} \sim \sD_{c_x}^M}
            \pb{
                \ell\pa{\phi(f(x), f(x^+), \pc{f(u_n)}_n, \pc{f(v_m)}_n)}
            }.
\end{align}
In Equation~\ref{eq:dcl_reexpress_using_loss}, $\phi:\R^{D(M+N+2)}\to\R^{M+N+2}$ computes some statistics of the representations
\begin{align}
    \label{eq:phi_def}
    \phi(\overline{x}, \overline{x}^+, \pc{\overline{u}_n}, \pc{\overline{v}_m})
        \triangleq \pa{
            \pc{\overline{x}^T(\overline{u}_n - \overline{x}^+)}_n, 
            \pc{\overline{x}^T(\overline{v}_m - \overline{x}^+)}_m, 
            \overline{x}^T \overline{x}^+, 
            \kappa(\overline{x})
        },
\end{align}
In Equation~\ref{eq:dcl_reexpress_using_loss}, $\ell:\R^{N+M+2}\to\R$ is the loss function
\begin{align}
    \label{eq:dcl_ell}
    \ell(\pc{a_n}, \pc{b_m}, c, d)
        \triangleq \log\pb{
            1 + N \max\pa{
                \frac{1}{1-d} \pa{
                    \frac{1}{N} \sum_{n=1}^N e^{a_n} - d \frac{1}{M} \sum_{m=1}^M e^{b_m}
                }
                ,
                e^{-1-c}
            }
        }.
\end{align}
We can verify that Equation~\ref{eq:dcl_reexpress_using_loss} is correct by noting that the integrand in Equation~\ref{eq:dcl_depends_on_fNM} can be expressed using a composition of $\ell$ and $\phi$
\begin{align}
    &-\log\pb{
        \frac{
            e^{f(x)^Tf(x^+)}
        }{
            e^{f(x)^Tf(x^+)} + Ng(x, \pc{u_n}, \pc{v_m})
        }
    }
        = \log\pb{
            1 + N \frac{ g(x,\pc{u_n},\pc{v_m}) }{ e^{f(x)^Tf(x^+)} }
        } \\
        &\quad= \log\Bigg[
            1 + N \max\Bigg\{
                \frac{1}{1-\kappa(f(x))} \pa{
                    \frac{1}{N}\sum_{n=1}^N e^{f(x)^T(f(u_n) - f(x^+))} - \kappa(f(x)) \frac{1}{M} \sum_{m=1}^M e^{f(x)^T(f(v_m) - f(x^+))}
                } \\ 
                &\quad\quad\quad\quad\quad\quad\quad\quad\quad\quad,
                e^{-1-f(x)^Tf(x^+)}
            \Bigg\}
        \Bigg] \\
        &\quad= \ell\pa{\phi(f(x), f(x^+), \pc{f(u_n)}_n, \pc{f(v_m)}_n)}.
\end{align}

Given a data set $\sS \triangleq (x_t, x_t^+, \pa{u_{tn}}_n, \pc{v_{tm}}_m) \in \sD_{\text{sim}}\times\sD^N \times \sD_{c_{x_t}}^M$ with $T$ samples, the empirical estimate of the debiased contrastive loss in Equation~\ref{eq:dcl_reexpress_using_loss} is
\begin{align}
    \hat{\sL}(f, N, M)
        \triangleq \frac{1}{T} \sum_{t=1}^T \ell\pa{\phi(f(x_t), f(x_t^+), \pc{f(u_{tn})}_n, \pc{f(v_{tm})}_m)}.
\end{align}
The learning process finds a representation from a function class $\sF$ with bounded norm, i.e., $\sF \triangleq \pc{f:\sX\to\R^D \mid \norm{f(\cdot)} \leq 1}$. For example, $\sF$ can contain all functions that map to unit hypersphere. Specifically, the algorithm finds the encoder function $\hat{f}$ via empirical risk minimization, i.e., $\hat{f} \in \argmin_{f\in\sF} \hat{\sL}(f, N, M)$.

The following Lemma is a more specific version of Lemma A.2 in \citet{saunshiTheoreticalAnalysisContrastive2019}, to account for the differences in our definition of $\phi,\ell$ The goal is to bound the debiased contrastive loss for the risk minimizer $\hat{f}$. 

First we introduce some notations. $\sigma$ are Rademacher random variables. The restriction of $f$ onto $\sS$ is
\begin{align}
    f_{|\sS} \triangleq (  f_d(x_t), f_d(x_t^+), \pc{f_d(u_n)}_n, \pc{f_d(v_m)}_m )_{d\in\pb{D},t\in\pb{T}}.
\end{align}
$\sR_{\sS}(\sF)$ is the empirical Rademacher complexity of the function class $\sF$ with respect to the training data $\sS$, i.e., 
\begin{align}
    \sR_{\sS}(\sF) \triangleq \E_{\sigma\sim\pc{\pm 1}^{(M+N+2)DT}} \pb{\sup_{f\in\sF} \langle \sigma, f_{|\sS} \rangle }.
\end{align}
$h:\R^{(M+N+2)D}\to \R$ as $h \triangleq \ell \circ \phi$ is a function that computes the loss given the embeddings.

\begin{lemma}
    \label{lem:fhat_f_dcl_bounds}
    If $h$ is $L_h$-Lipschitz and $\ell$ bounded by $B$. With probability $1-\delta$, for all $f\in\sF$
    \begin{align}
        \sL(\hat{f}, N, M)
            \leq \sL(f, N, M) + \sO\pa{
                \frac{L_h \sR_{\sS}(\sF)}{T} + B \sqrt{\frac{\log\pa{\frac{1}{\delta}}}{T}}
            }.
    \end{align}
    Specifically,
    \begin{align}
        B
            &= \sO\pa{
                \log N 
            }
    \end{align}
    and $L_h = L_{\ell} \cdot L_{\phi}$ where
    \begin{align}
        L_{\ell} 
            &= \sO\pa{
                \sqrt{ \frac{1}{(1-\eta_{\text{max}})^2 N }  + \frac{ \eta_{\text{max}}^2 }{ (1-\eta_{\text{max}})^2 M } + \frac{1}{(1-\eta_{\text{max}})^4} }
            },
            \label{eq:lip_constant_ell}
            \\ 
        L_{\phi}
            &= \sO\pa{
                \sqrt{ N+M + \norm{\nabla \kappa}_2^2 }
            }.
    \end{align}
\end{lemma}

\begin{proof}
    The first part of the Lemma is simply Lemma A.2 in \citet{saunshiTheoreticalAnalysisContrastive2019} without computing the Lipschitz constant for their definition of $\phi'(\overline{x}, \overline{x}^+, \pc{\overline{x}_n^-}_n) = (\overline{x}^T(\overline{x}_n^- - \overline{x}^+) )_n$. We refer reader to \citet{saunshiTheoreticalAnalysisContrastive2019} for more details.

    Note $\ell$ in Equation~\ref{eq:dcl_ell} is lower bounded by $\log(1)=0$. Note $a_n,b_m \in [-2,2]$, $c\in[-1,1]$, and $d\in[\eta_{\text{min}}, \eta_{\text{max}}]$. By plugging in $a_n = 2$, $b_m=2$, $c=-1$, and $d = \eta_{\text{max}}$, the upper bound of $\ell$ is 
    \begin{align}
        \ell(\pc{a_n},\pc{b_m}, c, d)
            \leq \log \pa{
                1 + N \max\pa{
                    \frac{1}{1-\eta_{\text{max}}} \pa{ e^2 - \eta_{\text{max}} e^{-2} },
                    1
                }
            }
            = \sO(\log N).
    \end{align}
    Therefore, $B = \sO(\log N)$. 

    To compute the Lipschitz constant for $\ell$, we express $\ell$ as composition of two functions $\psi,\omega$, i.e., $\ell = \omega \circ \psi$ where $\psi: \R^{N+M+2}\to\R$ and $\omega:\R\to\R$ are defined as
    \begin{align}
        \omega(z) 
            &= \log(1+Nz), \\ 
        \psi(\pc{a_n}, \pc{b_m}, c, d)
            &= \max\pa{
                \frac{1}{1-d}\pa{
                    \frac{1}{N} \sum_{n=1}^N e^{a_n} - d \frac{1}{M} \sum_{m=1}^M e^{b_m}
                },
                e^{-1-c}
            }.
        \label{eq:psi_def}
    \end{align}
    Note $z \geq e^{-1-c} \geq e^{-2}$ due to max in Equation~\ref{eq:psi_def}. Therefore, we can bound the Lipschitz constant for $\omega$ with
    \begin{align}
        \label{eq:lip_constant_omega}
        L_{\omega}
            \leq \left| \frac{d\omega}{dz} \right|
            \leq \frac{N}{1 + Ne^{-2}}
            \leq e^{-2}.
    \end{align}
    To bound the Lipschitz constant for $\psi$, we compute the Frobenius norm of the Jacobian $J_{\psi}$ of $\psi$. We use facts $a_n,b_m\in[-2,2]$, $c\in[-1,1]$, and $d\in[\eta_{\text{min}}, \eta_{\text{max}}]$ to derive upper bounds to the absolute value of the partial derivatives:
    \begin{align}
        \left| \frac{\partial \psi}{\partial a_n} \right| 
            &\leq \frac{e^{a_n}}{(1-d)N} 
            \leq \frac{e^2}{(1-\eta_{\text{max}})N} 
                \quad\quad n\in\pb{N} \\ 
        \left| \frac{\partial\psi}{\partial b_m} \right| 
            &\leq \frac{d \cdot e^{b_m} }{(1-d)M}
            \leq \frac{\eta_{\text{max}} e^2}{(1-\eta_{\text{max}})M} 
                \quad\quad m\in\pb{M} \\ 
        \left| \frac{\partial\psi}{\partial c} \right|
            & = \frac{1}{(1-d)^2} \pa{
                \frac{1}{N}\sum_{n=1}^N e^{a_n} - \frac{1}{M}\sum_{m=1}^M e^{b_m} 
            }
            \leq \frac{e}{(1-\eta_{\text{max}})^2}  \\ 
        \left|  \frac{\partial\psi}{\partial d} \right|
            &\leq e^{-1-c} \leq 1.
    \end{align}
    We use the Frobeninus norm of the Jacobnian to bound the Lipschitz constant
    \begin{align}
        \label{eq:lip_constant_psi}
        L_{\psi} 
            \leq \norm{J_{\psi}}_2  
            \leq \norm{J_{\psi}}_F
            \leq \sqrt{
                \sum_{n=1}^N \frac{e^4}{ (1-\eta_{\text{max}})^2 N^2 } + 
                \sum_{m=1}^M \frac{ \eta_{\text{max}}^2 }{ (1-\eta_{\text{max}})^2 M^2 } + 
                \frac{ e^2 }{ (1-\eta_{\text{max}})^4 } + 
                1
            }.
    \end{align}
    Combine Lipschitz constants $L_{\omega}$ in Equation~\ref{eq:lip_constant_omega} and $L_{\psi}$ in Equation~\ref{eq:lip_constant_psi}, a Lipschitz constant $L_{\ell} = L_{\omega}\cdot L_{\psi}$ for $\ell$ in Equation~\ref{eq:lip_constant_ell} follows directly.
     
    To compute a Lipschitz constant for $\phi$, we use similar tactic as previous by computing the Jacobnian $J_{\phi}$ of $\phi$. We first expand out the inner products in Equation~\ref{eq:phi_def},
    \begin{align}
        \label{eq:phi_def_expanded}
        \phi(\overline{x}, \overline{x}^+, \pc{\overline{u}_n}, \pc{\overline{v}_m})
            = \pa{
                \pc{ \sum_{d} \overline{x}_d(\overline{u}_{dn} - \overline{x}^+_{d}) }_n, 
                \pc{ \sum_{d} \overline{x}_d(\overline{v}_{dm} - \overline{x}^+_{d}) }_m, 
                \sum_d \overline{x}_d \overline{x}_d^+, 
                \kappa(\overline{x})
            }.
    \end{align}
    The partial derivatives of this vector-valued function $\phi$, given $d\in\pb{D}$, are
    \begin{align}
        \frac{\partial \phi_i}{ \partial \overline{x}_d }
            &= \begin{cases}
                (\overline{u}_{di} - \overline{x}^+_d) & i\in\pb{N} \\
                (\overline{v}_{d(i-N)} - \overline{x}^+_d) & i=N+1,\cdots,N+M \\
                \overline{x}^+_d & i = N+M+1 \\
                \frac{\partial \kappa}{\partial \overline{x}_d} & i = N+M+2
            \end{cases} \\ 
        \frac{\partial \phi_i}{\partial \overline{x}^+_d}
            &= \begin{cases}
                -\overline{x}_d & i \in [N+M] \\ 
                \overline{x}_d & i = N+M+1 \\
                0 & \text{o.w.}
            \end{cases} \\ 
        \frac{\partial \phi_i}{\partial \overline{u}_{dj} }
            &= \begin{cases}
                \overline{x}_d & i = j, j \in \pb{N} \\ 
                0 & \text{o.w.} \\ 
            \end{cases} \\ 
        \frac{\partial \phi_i}{\partial \overline{v}_{dj} }
            &= \begin{cases}
                \overline{x}_d & i = j + N, j \in [M] \\ 
                0 & \text{o.w.} \\ 
            \end{cases}.
    \end{align}
    We use the Frobenius norm of the Jacobian to bound the Lipschitz constant
    \begin{align}
        L_{\phi}
            &\leq \norm{J_{\phi}}_2 
            \leq \norm{J_{\phi}}_F \\ 
            &\leq \sqrt{
                \begin{aligned}
                    \sum_{i=1}^N \sum_d (\overline{u}_{di} - \overline{x}_d^+ )^2 + 
                    \sum_{i=N+1}^{N+M} \sum_d (\overline{v}_{d(i-N)} - \overline{x}_d^+)^2 + 
                    \sum_d (\overline{x}_d^+)^2 +  \\ 
                    \sum_d \pa{ \frac{\partial \kappa}{\partial \overline{x}_d} }^2 + 
                    (N+M) \sum_d (\overline{x}_d)^2 + 
                    \sum_d (\overline{x}_d)^2 + 
                    N\sum_d (\overline{x}_d)^2 + 
                    M\sum_d (\overline{x}_d)^2 
                \end{aligned}
            } \\ 
            &\leq \sqrt{
                4N + 4M + 1 + \norm{\nabla \kappa}_2^2 + (N+M) + 1 + N + M
            } \\ 
            &= \sO\pa{
                \sqrt{ N+M + \norm{\nabla \kappa}_2^2 }
            }
    \end{align}
    Here, norm of representations is bounded, i.e., $\norm{\overline{x}}_2,\norm{\overline{x}^+}_2,\norm{\overline{u}_n}_2, \norm{\overline{v}_m}_2 \leq 1$. 

    We can compute a Lipschitz constant for $h$ by multiplying that for $\ell,\phi$, i.e., $L_h = L_{\ell} \cdot L_{\phi}$. 
\end{proof}

The following theorem provides the generalization bounds on classification tasks. 

\begin{theorem}
    \label{thm:cls_generalization_bounds}
    With probability at least $1-\delta$, for all $f\in\sF$, and $N\geq \frac{1-\rho_{\text{min}}}{\rho_{\text{min}}}$, 
    \begin{align}
        \sL_{\text{sup}}(\hat{f})
            \leq \sL(f, N, M) + 
            \sO\pa{
                \sqrt{ \frac{1}{M} } + \sqrt{ \frac{1}{N} } + \E_{x\sim\sD}\pb{\left| \frac{1}{1-\eta(x)} - \frac{1}{1-\rho(c_x)} \right|} + \frac{L_h \sR_{\sS}(\sF)}{T} + B\sqrt{ \frac{\log(\frac{1}{\delta})}{T} }
            }
    \end{align}
\end{theorem}

\begin{proof}
    Simply combine the results that we proved so far gives us the result
    \begin{align}
        \sL_{\text{sup}}(\hat{f})
            &\leq \overline{\sL}(\hat{f}, N) 
                \tag{ Use Lemma~\ref{lem:sup_vs_dcl_loss_bounds} } \\ 
            &\leq \sL(\hat{f}, N, M) + \sO\pa{
                \sqrt{ \frac{1}{M} } + \sqrt{ \frac{1}{N} } + \E_{x\sim\sD}\pb{\left| \frac{1}{1-\eta(x)} - \frac{1}{1-\rho(c_x)} \right|}
            }
                \tag{ Use Proposition~\ref{prop:dcl_estimation_error} } \\ 
            &\leq \sL(f, N, M) + \sO\pa{
                \sqrt{ \frac{1}{M} } + \sqrt{ \frac{1}{N} } + \E_{x\sim\sD}\pb{\left| \frac{1}{1-\eta(x)} - \frac{1}{1-\rho(c_x)} \right|} + \frac{L_h \sR_{\sS}(\sF)}{T} + B\sqrt{ \frac{\log(\frac{1}{\delta})}{T} }
            }
                \tag{ Use Lemma~\ref{lem:fhat_f_dcl_bounds} }.
    \end{align}
\end{proof}

This bound states that if (i) the function class $\sF$ is sufficiently rich (i.e., contains some encoder function for which $\sL(f, N, M)$ is small), (ii) trained over large data (i.e., $T$ is large), (iii) the number of negatives sampled is large (i.e., large $N,M$), and (iv) $\eta$ provides a good estimate of sample-specific class probabilities, then the encoder $\hat{f}$ will perform well on the downstream classification tasks.

\section{Additional Details for Experiments}

\begin{table}[ht]
    \centering
    \begin{tabular}{l|ll}
    \toprule
    \textbf{Label} & \textbf{Negative Prompt} & \textbf{Positive Prompt} \\
    \midrule
    pneumonia       & ``No signs of pneumonia''         & ``Findings suggesting pneumonia'' \\
    cardiomegaly    & ``Heart size normal''             & ``Cardiomegaly'' \\
    edema           & ``No pulmonary edema''            & ``Pulmonary edema'' \\
    effusion        & ``No pleural effusions''          & ``Pleural effusions'' \\
    pneumothorax    & ``No pneumothorax''               & ``Pneumothorax'' \\
    \bottomrule
    \end{tabular}
    \caption{Zero-shot image classification prompts.}
    \label{tab:zs_cls_prompts}
    \end{table}

\section{Additional Results}
\label{appendix:additional_results}

\begin{figure}
    \centering 
    \includegraphics[width=.9\textwidth]{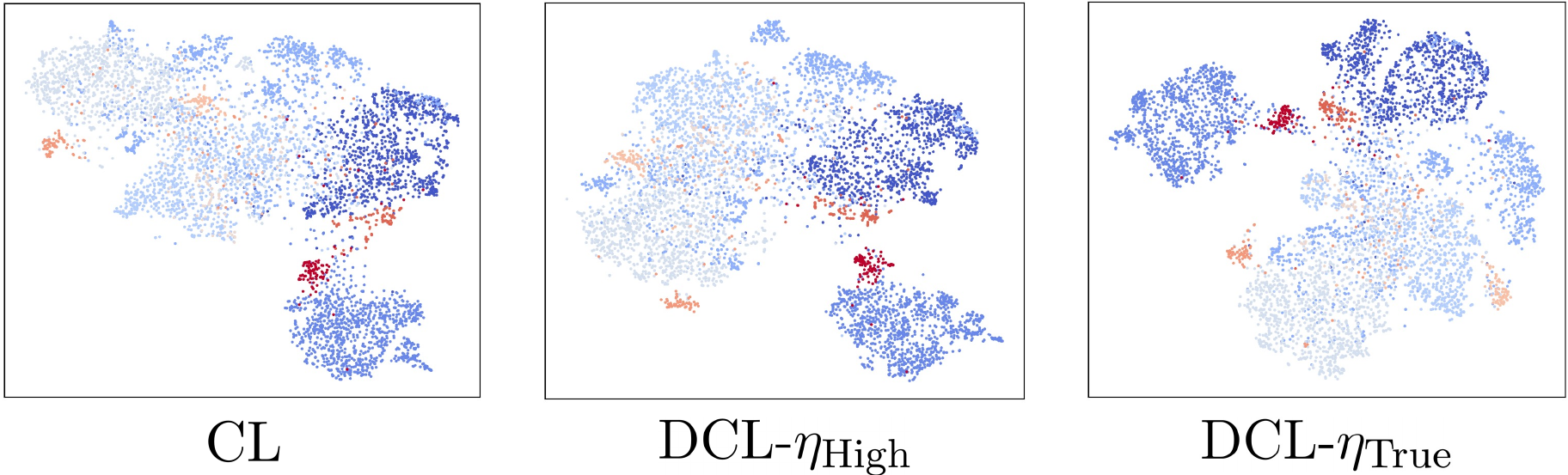}
    \caption{t-SNE visualization of learned representations on CIFAR10-0.1. Subsampled classes are in shades of red and the remaining classes are in shades of blue. Using the true sample-specific class probability function (DCL-$\eta_{\text{True}}$) leads to better class separation than both contrastive learning (CL) and using the incorrect sample-specific class probability function (DCL-$\eta_{\text{High}}$), especially for the subsampled classes. }
    \label{fig:cifar10_qualitative}  
  \end{figure}

\begin{table}[ht]
    \centering
    \begin{tabular}{L{0.14\linewidth}R{0.11\linewidth}R{0.11\linewidth}R{0.11\linewidth}R{0.11\linewidth}R{0.11\linewidth}R{0.11\linewidth}p{0pt}}
    \toprule
    Method & Cardiomegaly & Edema & Effusion & Pneumonia & Pneumothorax & Avg &  \\
    \midrule
    BioViL & 0.752 & 0.887 & \textBF{0.914} & 0.768 & 0.594 & 0.783 &  \\
    LSE+NL & \textBF{0.888} & 0.876 & 0.798 & 0.853 & 0.526 & 0.788 &  \\
    \scriptsize \quad w/ DCL-$\eta_{\text{LM}}$ & 0.884 & \textBF{0.889} & 0.913 & \textBF{0.861} & \textBF{0.760} & \textBF{0.862} &  \\
    \bottomrule
    \end{tabular}
    \caption{Zero-shot image classification performance. We report the AUC for 5 common radiology findings and their average (Avg) on CheX5 data set. Image-text model trained with our method outperforms state-of-the-art baselines. }
    \label{tab:sota_zs_cls}
\end{table} 
    
\begin{table}[ht]
    \centering
    \begin{tabular}{L{0.16\linewidth}C{0.12\linewidth}C{0.12\linewidth}p{0pt}}
    \toprule
    Method & CNR$\uparrow$ & mIoU$\uparrow$ &  \\
    \midrule
    BioViL & 1.142 & 0.174 &  \\
    LSE+NL & 1.400 & 0.190 &  \\
    \scriptsize \quad w/ DCL-$\eta_{\text{LM}}$ & \textBF{1.486} & \textBF{0.195} &  \\
    \bottomrule
    \end{tabular}
    \caption{Visual grounding performance. We report a measure of the discrepancy between region-sentence scores inside and outside the ground truth bounding box (contrast-to-noise ratio or CNR) and the mean IoU of a thresholded region-sentence map and the ground truth bounding box over a set of threshold (mIoU). Image-text model trained with our method outperforms BioViL and LSE+NL on both measures.}
    \label{tab:sota_grounding}
\end{table}

\begin{table}[ht]
\centering
\begin{tabular}{L{0.140\linewidth}C{0.075\linewidth}C{0.075\linewidth}C{0.075\linewidth}C{0.075\linewidth}|C{0.075\linewidth}C{0.075\linewidth}C{0.075\linewidth}C{0.075\linewidth}p{0pt}}
    \toprule
Method & \multicolumn{4}{c}{\quad Image $\rightarrow$ Text} & \multicolumn{4}{c}{\quad Text $\rightarrow$ Image} &  \\
& \scriptsize R@10 $\uparrow$& \scriptsize R@50 $\uparrow$& \scriptsize R@100 $\uparrow$& \scriptsize MedR$\downarrow$ & \scriptsize R@10 $\uparrow$& \scriptsize R@50 $\uparrow$& \scriptsize R@100 $\uparrow$& \scriptsize MedR$\downarrow$ &  \\
\midrule
BioViL & 0.07 & 0.26 & 0.40 & 151 & 0.08 & 0.26 & 0.40 & 146 &  \\
LSE+NL & 0.09 & 0.29 & 0.44 & 123 & \textBF{0.10} & 0.32 & 0.49 & 107 &  \\
\scriptsize \quad w/ DCL-$\eta_{\text{LM}}$ & \textBF{0.10} & \textBF{0.31} & \textBF{0.48} & \textBF{106} & \textBF{0.10} & \textBF{0.33} & \textBF{0.50} & \textBF{102} &  \\
\bottomrule
\end{tabular}
\caption{Cross-modal retrieval performance. We report recall for the top 10, 50 and 100 answers returned by the method (R@k), as well as the median rank of the ground truth element for sentence retrieval based on region queries and for region retrieval based on sentence queries (MedR). Our method improves the baselines.}
\label{tab:sota_retrieval}
\end{table}

\begin{figure}[ht]
    \centering
    \begin{verbatim}
        # pos: exponential for positive example
        # neg: sum of exponentials for negative examples
        # N : number of negative examples
        # eta : fixed class probability
        # p: likelihood of the text scored by a language model
        # a, k: hyperparameter that maps p to sample-specific class probability
        
        standard_loss = -log(pos / (pos + neg))
        dcl_loss = -log(pos / (pos + (neg - N * eta * pos) / (1 - eta)))
        eta_LM = a * p**k
        dcl_LM_loss = -log(pos / (pos + (neg - N * eta_LM * pos) / (1 - eta_LM)))
    \end{verbatim}
    \caption{Pseudocode for our proposed sample-specific class probability estimate for debiased contrastive learning with $M=1$. The likelihood of the text scored by a language model can be precomputed and yield negligible computational overhead. Implementing our proposed changes requires minimal modification to code.}
    \label{fig:pseudocode}
\end{figure}

\end{document}